%% file: main.tex
\documentclass{article}

\usepackage{microtype}
\usepackage{graphicx}
\usepackage{booktabs} % for professional tables
\usepackage{enumitem}
\usepackage{xcolor}
\usepackage{amsmath,amsthm}
\usepackage{amssymb}
\usepackage{dsfont}
\usepackage{subcaption}

\usepackage{hyperref}

\DeclareMathOperator{\var}{Var}

\newcommand{\be}{\mathbf{e}}
\newcommand{\bg}{\mathbf{g}}
\newcommand{\bu}{\mathbf{u}}

\newcommand{\bw}{\mathbf{w}}
\newcommand{\bx}{\mathbf{x}}
\newcommand{\bz}{\mathbf{z}}
\newcommand{\calA}{\mathcal{A}}
\newcommand{\calB}{\mathcal{B}}
\newcommand{\calD}{\mathcal{D}}
\newcommand{\calI}{\mathcal{I}}
\newcommand{\calN}{\mathcal{N}}

\newcommand{\calZ}{\mathcal{Z}}

\newtheorem{theorem}{Theorem}

\usepackage[accepted]{icml2022}

\icmltitlerunning{Bounding Training Data Reconstruction in Private (Deep) Learning}

\begin{document}

\twocolumn[
\icmltitle{Bounding Training Data Reconstruction in Private (Deep) Learning}



\icmlsetsymbol{equal}{*}

\begin{icmlauthorlist}
\icmlauthor{Chuan Guo}{metaai}
\icmlauthor{Brian Karrer}{meta}
\icmlauthor{Kamalika Chaudhuri}{metaai}
\icmlauthor{Laurens van der Maaten}{metaai}
\end{icmlauthorlist}

\icmlaffiliation{metaai}{Meta AI}
\icmlaffiliation{meta}{Meta}

\icmlcorrespondingauthor{Chuan Guo}{chuanguo@fb.com}

\icmlkeywords{Machine Learning, ICML}

\vskip 0.3in
]



\printAffiliationsAndNotice{}  % leave blank if no need to mention equal contribution

\begin{abstract}
Differential privacy is widely accepted as the \emph{de facto} method for preventing data leakage in ML, and conventional wisdom suggests that it offers strong protection against privacy attacks. However, existing semantic guarantees for DP focus on membership inference, which may overestimate the adversary's capabilities and is not applicable when membership status itself is non-sensitive. In this paper, we derive semantic guarantees for DP mechanisms against training data reconstruction attacks under a formal threat model. We show that two distinct privacy accounting methods---R\'{e}nyi differential privacy and Fisher information leakage---both offer strong semantic protection against data reconstruction attacks.
\end{abstract}

\input{sections/intro}
\input{sections/background}
\input{sections/setup}
\input{sections/rdp}
\input{sections/fil}

\input{sections/fil_sgd}
\input{sections/experiments}
\input{sections/discussion}

\section*{Acknowledgements}
We thank Mark Tygert and Sen Yuan for helping us realize the connection between RDP and FIL, and Alban Desmaison and Horace He for assistance with the code.

\bibliography{citations}
\bibliographystyle{icml2022}

\newpage
\appendix
\onecolumn
\input{sections/appendix}

\end{document}

%% file: sections/intro.tex
\section{Introduction}
\label{sec:intro}

Machine learning models are known to memorize their training data. This vulnerability can be exploited by an adversary to compromise the privacy of participants in the training dataset when given access to the trained model and/or its prediction interface~\cite{fredrikson2014privacy, fredrikson2015model, shokri2017membership, carlini2019secret}.
%Indeed, prior works have shown that private information such as membership status~\cite{shokri2017membership, yeom2018privacy, salem2018ml}, sensitive attributes~\cite{fredrikson2014privacy, yeom2018privacy}, and even near-exact training samples~\cite{fredrikson2015model, zhang2020secret, carlini2019secret, carlini2021extracting} can be extracted from a trained model under practical scenarios.
By far the most accepted mitigation measure against such privacy leakage is differential privacy (DP;~\citet{dwork2014algorithmic}), which upper bounds the information contained in the learner's output about its training data via statistical divergences.
%and asserts that an adversary cannot gain much \emph{additional} information about the private training dataset by inspecting the trained model.
However, such a \emph{differential guarantee} is often hard to interpret, and it is unclear how much privacy leakage can be tolerated for a particular application~\cite{jayaraman2019evaluating}.

Recent studies have derived \emph{semantic guarantees} for differential privacy, that is, \emph{how does the private mechanism limit an attacker's ability to extract private information from the trained model?} For example, \citet{yeom2018privacy} showed that a differentially private learner can reduce the success rate of a membership inference attack to close to that of a random coin flip. Semantic guarantees serve as more interpretable translations of the DP guarantee and provide reassurance of protection against privacy attacks.
However, existing semantic guarantees focus on protection of membership status, which has several limitations: 1. There are many scenarios where membership status itself is not sensitive but the actual data value is, \emph{e.g.}, census data and location data. 2. It only bounds the leakage of the binary value of membership status as opposed to \emph{how much} information can be extracted. 3. Membership inference is empirically much easier than powerful attacks such as training data reconstruction~\cite{carlini2019secret, zhang2020secret, balle2022reconstructing},
%Data reconstruction is empirically much harder than membership inference~\cite{balle2021reconstructing},
and hence it may be possible to provide a strong semantic guarantee against data reconstruction attacks even when membership status cannot be protected.

In this work, we focus on deriving semantic guarantees against \emph{data reconstruction attacks} (DRA), where the adversary's goal is to reconstruct instances from the training dataset. 
Under mild assumptions, we show that if the learning algorithm is $(2,\epsilon)$-R\'{e}nyi differentially private, then the expected mean squared error (MSE) of an adversary's estimate for the private training data can be lower bounded by $\Theta(1/(e^\epsilon - 1))$. When $\epsilon$ is small, this bound suggests that the adversary's estimate incurs a high MSE and is thus unreliable, in turn guaranteeing protection against DRAs.%In particular, we compare the MSE lower bound with semantic guarantees for membership inference attacks, and show that certain values of $\epsilon$ confer non-trivial protection against data reconstruction attacks even when the membership inference guarantee is vacuous.

Furthermore, we show that a recently proposed privacy framework called \emph{Fisher information leakage} (FIL;~\citet{hannun2021measuring}) can be used to give tighter semantic guarantees for common private learning algorithms such as output perturbation~\cite{chaudhuri2011differentially} and private SGD~\cite{song2013stochastic, abadi2016deep}.
Importantly, FIL gives a \emph{per-sample} estimate of privacy leakage for every individual in the training set, and we empirically show that this per-sample estimate is highly correlated with the sample's vulnerability to data reconstruction attacks.
Finally, FIL accounting gives theoretical support for the observation that existing private learning algorithms do not leak much information about the vast majority of its training samples despite having a high privacy parameter.
%We show that for private training of generalized linear models and neural networks on benchmark datasets, FIL accounting gives superior semantic guarantees compared to the RDP bound.
%We extend FIL accounting to the subsampled Gaussian mechanism~\cite{abadi2016deep} and show that when applied to SGD training of generalized linear models and neural networks with gradient perturbation, FIL accounting gives a tighter, per-sample lower bound of expected MSE for a data reconstruction adversary.

%% file: sections/background.tex
\section{Background}
\label{sec:background}

\begin{figure*}[t!]
\centering
\begin{subfigure}{.47\textwidth}
  \centering
  \includegraphics[width=\linewidth]{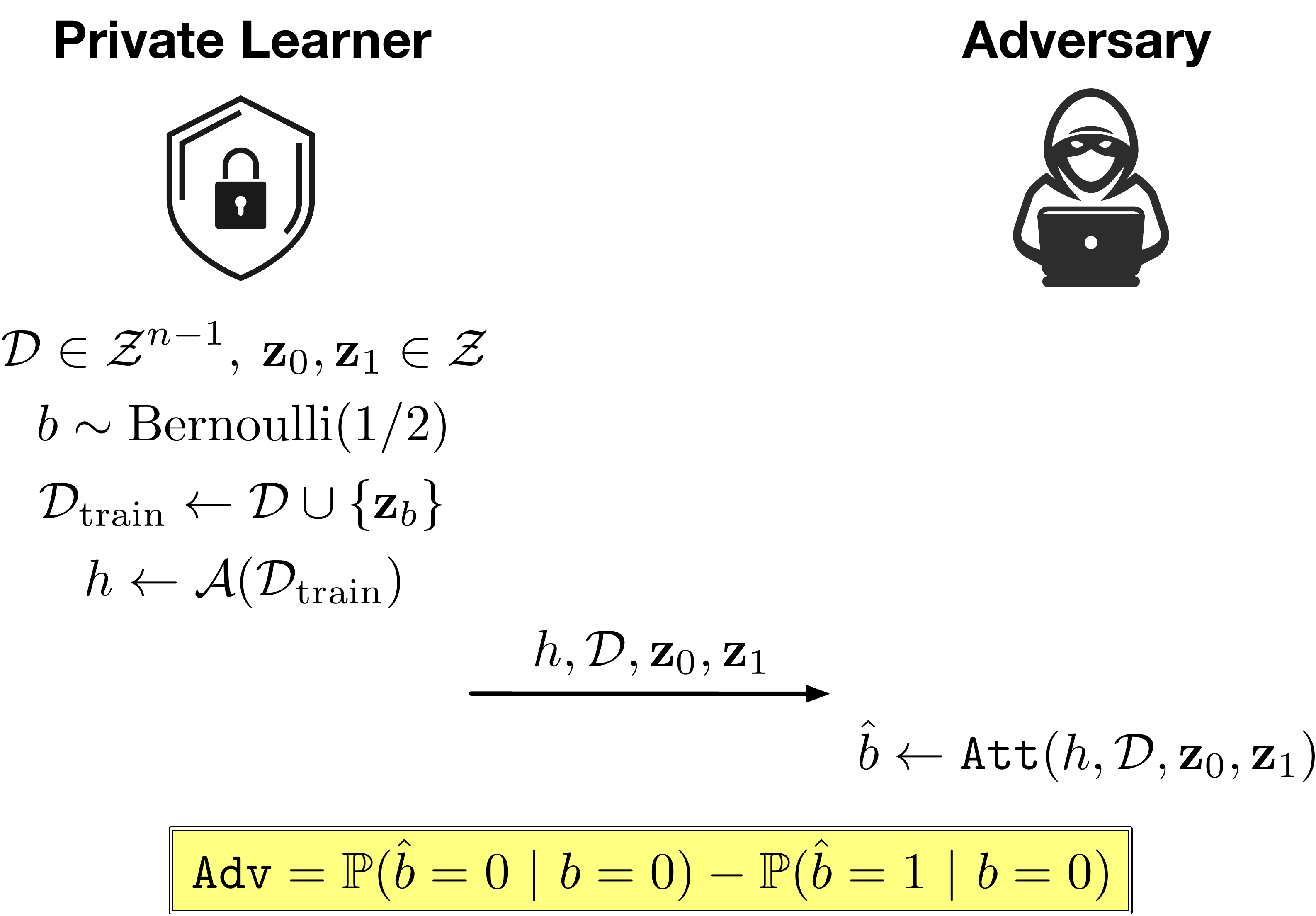}
  \caption{Membership inference attack game}
  \label{fig:mia_attack}
\end{subfigure}
\hspace{2ex}
\begin{subfigure}{.47\textwidth}
  \centering
  \includegraphics[width=0.92\linewidth]{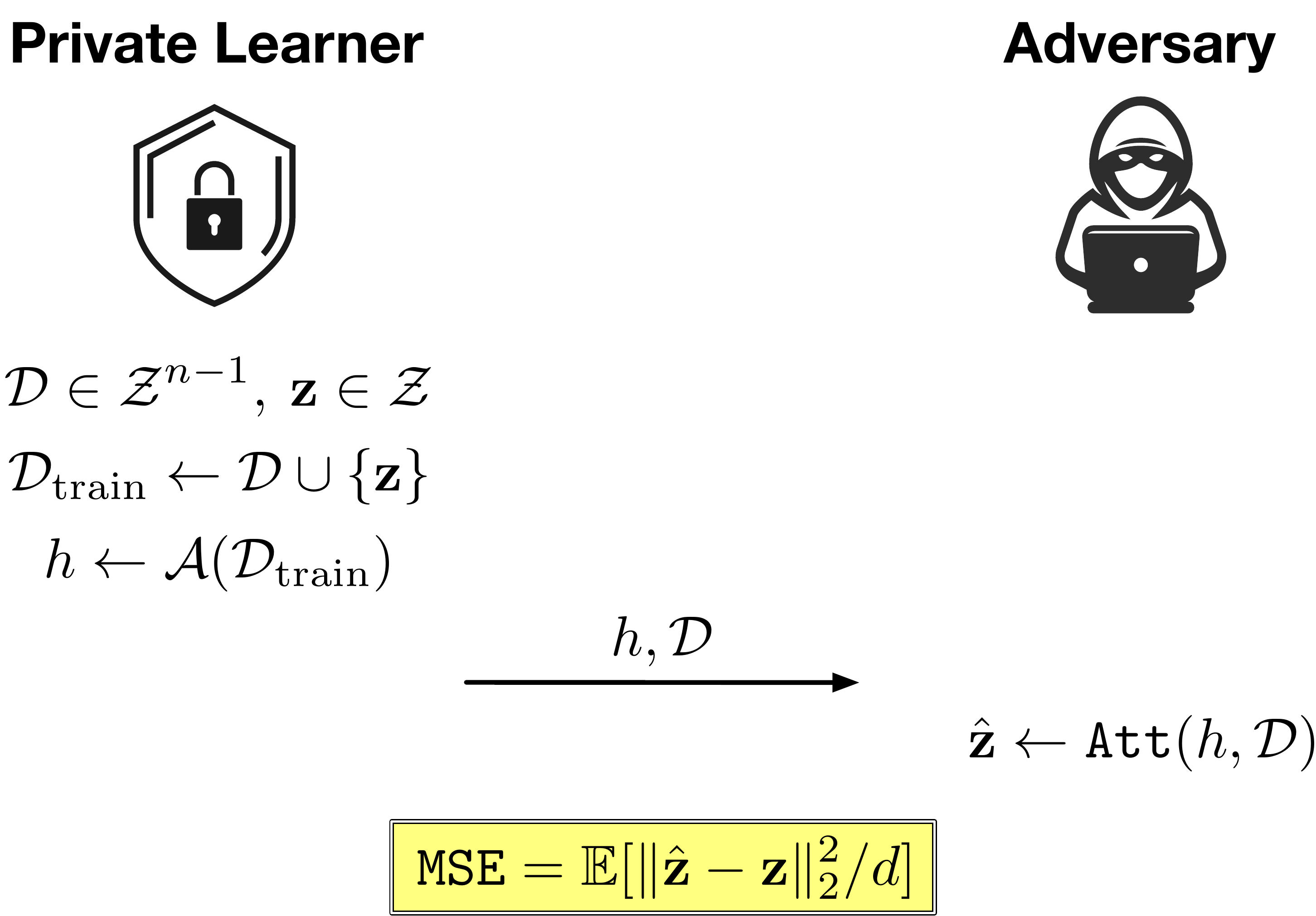}
  \caption{Data reconstruction attack game}
  \label{fig:dra_attack}
\end{subfigure}
\caption{Comparison of membership inference attacks (MIAs) and data reconstruction attacks (DRAs). Both attacks are formalized in terms of an attack game between a private learner and an adversary, and the metric of success is given in terms of advantage (\texttt{Adv}; higher is better) for MIA, and mean squared error (\texttt{MSE}; lower is better) for DRA.}
\label{fig:attack_comparison}
\end{figure*}

\paragraph{Data reconstruction attacks.} Machine learning algorithms often require the model to memorize parts of its training data~\cite{feldman2020does}, enabling adversaries to extract samples from the training dataset when given access to the trained model. Such \emph{data reconstruction attacks} (DRAs) have been carried out in realistic scenarios against face recognition models~\cite{fredrikson2015model, zhang2020secret} and neural language models~\cite{carlini2019secret, carlini2021extracting}, and constitute significant privacy risks for ML models trained on sensitive data.

\paragraph{Differential privacy.} The \emph{de facto} standard for data privacy in ML is \emph{differential privacy} (DP), %which prescribes a mathematical framework for controlling the leakage of information about individuals when performing data analysis. Differential privacy
which asserts that for adjacent datasets $\calD$ and $\calD'$ that differ in a single training sample, a model trained on $\calD$ is almost statistically indistinguishable from a model trained on $\calD'$, hence individual samples cannot be reliably inferred. Indistinguishability is measured using a statistical divergence $D$, and a (randomized) learning algorithm $\calA$ is differentially private if for any pair of adjacent datasets $\calD$ and $\calD'$, we have $D(\calA(\calD)~||~\calA(\calD')) \leq \epsilon$ for some fixed privacy parameter $\epsilon > 0$.
The most common choice for the statistical divergence $D$ is the \emph{max divergence}: $$D_\infty(P~||~Q) = \sup_{x \in \mathrm{supp}(Q)} \log \frac{P(x)}{Q(x)},$$ which bounds information leakage in the worst case and is the canonical choice for $\epsilon$-differential privacy~\cite{dwork2014algorithmic}. The weaker notion of $(\epsilon,\delta)$-DP uses the so-called ``hockey-stick'' divergence~\cite{polyanskiy2010channel}, which allows the max divergence bound to fail with probability at most $\delta > 0$~\cite{balle2018improving}. Another generalization uses the \emph{R\'{e}nyi divergence} of order $\alpha$~\cite{renyi1961measures}: $$D_\alpha(P~||~Q) = \frac{1}{\alpha - 1} \log \mathbb{E}_{x \sim Q}\left[\left( \frac{P(x)}{Q(x)}\right)^\alpha \right]$$ for $\alpha \in (1, \infty)$, and a learning algorithm $\calA$ is said to be $(\alpha, \epsilon)$-R\'{e}nyi differentially private (RDP;~\citet{mironov2017renyi}) if it is DP with respect to the $D_\alpha$ divergence.
Notably, an $(\alpha,\epsilon)$-RDP mechanism is also $(\epsilon + \log(1/\delta)/(\alpha-1), \delta)$-DP for any $0 < \delta < 1$~\cite{mironov2017renyi}, and RDP is the method of choice for composing multiple mechanisms such as in private SGD~\cite{song2013stochastic, abadi2016deep}. More optimal conversions between DP and RDP have been derived by \citet{asoodeh2021three}.
%Notably, R\'{e}nyi divergence coincides with the max divergence when $\alpha \nearrow \infty$, and coincides with the Kullback-Leibler (KL) divergence when $\alpha \searrow 1$~\cite{sason2016f}.

\paragraph{Semantic guarantees for differential privacy.} One challenge in applying differential privacy to ML is the selection of the privacy parameter $\epsilon$. For all statistical divergences, the distributions $\calA(\calD)$ and $\calA(\calD')$ are identical when $D(\calA(\calD)~||~\calA(\calD')) = 0$, hence a DP algorithm $\calA$ leaks no information about any individual when $\epsilon = 0$. However, it is not well-understood at what level of $\epsilon > 0$ does the privacy guarantee fail to provide any meaningful protection against attacks~\cite{jayaraman2019evaluating}.

Several works partially addressed this problem by giving semantic guarantees for DP against membership inference attacks (MIAs;~\citet{shokri2017membership, yeom2018privacy, salem2018ml}). In MIAs, the adversary's goal is to infer whether a given sample $\bz$ participated in the training set $\calD$ of a trained model. Formally, the attack can be modeled as a game between a learner and an adversary (see \autoref{fig:mia_attack}), where the membership of a sample is determined by a random bit $b$ and the adversary aims to output a prediction $\hat{b}$ of $b$.
The adversary's metric of success is given by the \emph{advantage} of the attack, which measures the difference between true and false positive rates of the prediction: $\texttt{Adv} = \mathbb{P}(\hat{b} = 0~|~b=0) - \mathbb{P}(\hat{b} = 1~|~b=0)$.
%In detail, suppose that $\calD_b = \calD \cup \{\bz_b\}$, where $b \sim \mathrm{Bernoulli}(1/2)$ and $\bz_0$, $\bz_1$ are samples from some data space $\calZ$. The attack algorithm $\texttt{Att}$ aims to distinguish between $\calD_0$ and $\calD_1$ by outputting a prediction $\hat{b}$ for $b$ when given a trained model $h \sim \calA(\calD_b)$.
\citet{humphries2020differentially} showed that if $\calA$ is $\epsilon$-DP, then $\texttt{Adv} \leq (e^\epsilon - 1) / (e^\epsilon + 1)$. Hence if $\epsilon$ is small, then the adversary cannot perform significantly better than random guessing. For instance, if $\epsilon = 0.1$ then the probability of correctly predicting the membership of a sample is at most $(\texttt{Adv} + 1) / 2 \approx 53\%$, which is negligibly better than a random coin flip. \citet{yeom2018privacy} and \citet{erlingsson2019we} derived similar results.

%% file: sections/setup.tex
\section{Formalizing Data Reconstruction Attacks}
\label{sec:setup}

\paragraph{Motivation.} Semantic guarantees for MIA can be useful for interpreting the protection of DP and selecting the privacy parameter $\epsilon$, but several issues remain:

1. Membership status is often not sensitive, but the underlying data value is. For example, a user's mobile device location can expose the user to unauthorized tracking, but its presence on the network is benign. In these scenarios, it is more meaningful to upper bound how much information can an adversary recover about a training sample.

2. Models trained on complex real-world datasets cannot achieve a low $\epsilon$ while maintaining high utility. \citet{tramer2020differentially} evaluated different private learning algorithms for training convolutional networks on the MNIST dataset, and showed that practically all current private learning algorithms require $\epsilon \geq 2$ in order to attain a reasonable level of test accuracy. At this $\epsilon$, the attacker's probability of correcting predicting membership becomes $>88\%$.

3. Data reconstruction is empirically much harder than MIA~\cite{balle2022reconstructing}, hence it may be possible to derive meaningful guarantees against DRAs even when the membership inference bound becomes vacuous.

\paragraph{Threat model.} Motivated by these shortcomings, we focus on formalizing data reconstruction attacks and deriving semantic guarantees against DRAs for private learning algorithms. \autoref{fig:dra_attack} defines the DRA game, which is a slight modification of the MIA game in \autoref{fig:mia_attack}. Let $\calZ$ be the data space, and suppose that the learner receives samples $\calD \in \calZ^{n-1}$ and $\bz \in \calZ$. Let $\calD_\text{train} = \calD \cup \{\bz\}$ be the training dataset, for which the randomized learner outputs a model $h \leftarrow \calA(\calD_\text{train})$ after training on $\calD_\text{train}$. The adversary receives $h$ and $\calD$ and runs the attack algorithm to obtain a reconstruction $\hat{\bz}$ of the sample $\bz$.

We highlight two major differences between the DRA game and the MIA game: 1. The attack target $\bz$ is unknown to the adversary. This change reflects the fact that the adversary's goal is to reconstruct $\bz$ given access to the trained model $h$, rather than infer the membership status of $\bz$. 2. The metric of success is $\texttt{MSE} = \mathbb{E}_h[\| \hat{\bz} - \bz \|_2^2/d]$, where $d$ is the data dimensionality. In other words, the adversary aims to achieve a low reconstruction MSE in expectation over the randomness of the learning algorithm $\calA$. Using MSE implicitly assumes that the underlying data is continuous and that the squared difference in $\hat{\bz} - \bz$ reflects semantic differences. While DRA motivates different metrics of success, we opt to measure MSE in our formulation.

%% file: sections/rdp.tex
\section{Error Bound From RDP}
\label{sec:rdp}

In this section, we show that any RDP learner implies a lower bound on the MSE of a reconstruction attack. %Our bound formally validates empirical observations that DP models offer protection against DRAs~\cite{carlini2019secret, balle2021reconstructing}.
Our crucial insight is to view the data reconstruction attack as a parameter estimation problem for the adversary: The sample $\bz$ induces a distribution over the space of models through the learning algorithm $\calA$. If we treat $\bz$ as the parameter of the distribution $\calA(\calD_\text{train})$\footnote{Under the assumptions outlined in \autoref{sec:setup}, all other parameters of this distribution, such as other training points in $\calD_\text{train}$ and hyperparameters, are known.}, we can then utilize statistical estimation theory to lower bound the estimation error of $\bz$ when given a single sample from the distribution $\calA(\calD_\text{train})$. 

Our main tool for proving this lower bound is the Hammersley-Chapman-Robbins bound (HCRB;~\citet{chapman1951minimum}), which we state and prove in \autoref{sec:crb_hcrb}. \autoref{thm:rdp_bound} below gives our MSE lower bound for RDP learning algorithms. Proof is given in \autoref{sec:proofs}.

\begin{theorem}
\label{thm:rdp_bound}
Let $\bz \in \calZ \subseteq \mathbb{R}^d$ be a sample in the data space $\calZ$, and let $\texttt{Att}$ be a reconstruction attack that outputs $\hat{\bz}(h)$ upon observing the trained model $h \leftarrow \calA(\calD_\text{train})$, with expectation $\mu(\bz) = \mathbb{E}_{\calA(\calD_\text{train})}[\hat{\bz}(h)]$. If $\calA$ is a $(2,\epsilon)$-RDP learning algorithm then:
\begin{equation*}
    \small
    \mathbb{E}\left[\|\hat{\bz}(h) - \bz\|_2^2 / d\right] \geq \underbrace{\frac{\sum_{i=1}^d \gamma_i^2 \mathrm{diam}_i(\calZ)^2/4d}{e^\epsilon - 1}}_\text{variance} + \underbrace{\frac{\|\mu(\bz) - \bz\|_2^2}{d}}_\text{squared bias},
\end{equation*}
where $\gamma_i = \inf_{\bz \in \calZ} |\partial \mu(\bz)_i / \partial \bz_i|$ and
$$\mathrm{diam}_i(\calZ) = \sup_{\bz, \bz' \in \calZ : \bz_j = \bz'_j \forall j \neq i} |\bz_i - \bz'_i|$$
is the diameter of $\calZ$ in the $i$-th dimension. In particular, if $\hat{\bz}(h)$ is unbiased then:
\begin{equation*}
    \mathbb{E}[\|\hat{\bz}(h) - \bz\|_2^2 / d] \geq \frac{\sum_{i=1}^d \mathrm{diam}_i(\calZ)^2/4d}{e^\epsilon - 1}.
\end{equation*}
\end{theorem}

\paragraph{Observations.} The semantic guarantee in \autoref{thm:rdp_bound} has several noteworthy features:

1. \emph{There is an explicit bias-variance trade-off for the adversary.} The adversary can control its bias-variance trade-off to optimize for MSE, with the trade-off factor determined by $\gamma_i$ and $\epsilon$. In essence, $\gamma_i$ measures how quickly the adversary's estimate $\hat{\bz}(h)$ changes with respect to $\bz$, and the lower bound \emph{degrades gracefully} with respect to this sensitivity.

2. \emph{The variance term is controlled by the privacy parameter $\epsilon$.} When $\epsilon=0$, all attacks have infinite variance, which reflects the fact that the adversary can only perform random guessing. As $\epsilon$ increases, the variance term decreases, hence the reconstruction attack can accurately estimate the underlying sample $\bz$.

%3. \emph{There is a paradoxical adversary that attains zero error.} A hypothetical adversary with perfect prior knowledge of $\bz$ can guess $\hat{\bz}(h) = \bz$. This estimator attains zero bias and has $\gamma_i = 0$ for all $i$, hence the error lower bound is zero! In fact, this paradoxical example is reassuring, since it is known that simultaneously achieving good model utility and preserving privacy is impossible in the presence of arbitrary adversary prior knowledge~\cite{dwork2010difficulties}.

3. \emph{The DRA bound can be meaningful even when MIA bounds may not be.} Suppose the input space is $\calZ = [0,100]$, then $\mathrm{diam}_1(\calZ) = 100$. At $\epsilon=2$, the unbiased bound evaluates to $10^4/(4 (e^\epsilon - 1)) \approx 391$, which means the adversary's estimate has standard deviation $\approx 19$, \emph{i.e.}, the adversary cannot be certain of their reconstruction up to $\pm 19$. This can be a very meaningful guarantee when the data is only semantically sensitive within a small range, \emph{e.g.}, age.

4. \emph{The bound also applies to $\epsilon$-DP.} R\'{e}nyi divergence is non-decreasing in its order $\alpha$~\cite{sason2016f}, \emph{i.e.}, $D_\alpha(P~||~Q) \leq D_\beta(P~||~Q)$ whenever $\alpha \leq \beta$, hence any $\epsilon$-DP mechanism satisfies \autoref{thm:rdp_bound} as well. Alternative, we can leverage tighter and more general conversions for $(\epsilon,\delta)$-DP~\cite{asoodeh2021three}.

\paragraph{Tightness.} The tightness of \autoref{thm:rdp_bound} has a significant dependence on $\mathrm{diam}_i(\calZ)$. Suppose that $\calZ = [0, M]$ for some $M > 0$, so $\mathrm{diam}_1(\calZ) = M$. Let $\calA(\calD_\text{train}) = z + \calN(0, \sigma^2)$ for any $z \in \calZ$, and let $\hat{z}(h) = h$ so that $\hat{z}$ is an unbiased estimator of $z$ with $\mathbb{E}[(\hat{z}(h) - z)^2] = \sigma^2$. It can be verified that $\calA$ satisfies $(2,\epsilon)$-RDP with $\epsilon = M^2/\sigma^2$, so \autoref{thm:rdp_bound} gives: $$\mathbb{E}[(\hat{z}(h) - z)^2] \geq \frac{M^2}{4(e^{M^2/\sigma^2} - 1)}.$$
As $M \rightarrow 0$, we have that:
$$\lim_{M \rightarrow 0} \frac{M^2}{4(e^{M^2/\sigma^2} - 1)} = \lim_{M \rightarrow 0} \frac{2M}{\frac{8M}{\sigma^2} e^{M^2/\sigma^2}} = \sigma^2 / 4,$$
so the bound is tight up to a constant factor. 
However, it is also clear that this bound converges to $0$ as $M \rightarrow \infty$, hence it can be arbitrarily loose in the worst case. We will show that Fisher information leakage---an alternative measure of privacy loss---can address this worst-case looseness.

%The above example shows a crucial limitation when deriving MSE lower bounds using divergences. This looseness comes from the worst-case characterization in DP---that $D(\calA(\calD \cup \{z_0\})~||~\calA(\calD \cup \{z_1\})) \leq \epsilon$ for the worst-case pair $z_0, z_1 \in \calZ$. For instance, the 2-R\'{e}nyi divergence $\epsilon = M^2/\sigma^2$ is a tight characterization of information leakage when distinguishing between $z_0=0$ and $z_1=M$. However, the same divergence bound is used when distinguishing between $z_0'=0$ and $z_1'=\Delta$ for any $\Delta > 0$, even though $z_0'$ and $z_1'$ are much less distinguishable compared to $z_0$ vs. $z_1$. 

%% file: sections/fil.tex
\section{Error Bound From FIL}
\label{sec:fil}

%\autoref{thm:rdp_bound} provides a lower bound on the MSE of a DRA adversary for an arbitrary RDP mechanism. However, this bound is not the tightest achievable bound for every RDP mechanism. In particular, we show that it is possible to develop tighter lower bounds for a private learner that uses SGD with Gaussian gradient perturbation---one of the most common uses of RDP accounting~\cite{abadi2016deep}. To do so, we adopt the Fisher information leakage (FIL) accountant from \citet{hannun2021measuring}. We show how to compute FIL for private SGD learners and use the resulting FIL to derive a tighter lower bound on the MSE of a DRA adversary attacking this private learner. 

\emph{Fisher information leakage} (FIL;~\citet{hannun2021measuring}) is a recently proposed framework for privacy accounting that is directly inspired by the parameter estimation view of statistical privacy. We show that FIL can be naturally adapted to give a tighter MSE lower bound compared to \autoref{thm:rdp_bound}.

\paragraph{Fisher information leakage.} Fisher information is a statistical measure of information about an underlying parameter from an observable random variable. Suppose that the learning algorithm $\calA$ produces a model $h \leftarrow \calA(\calD_\text{train})$ after training on $\calD_\text{train} = \calD \cup \{\bz\}$. The Fisher information matrix (FIM) of $h$ about the sample $\bz$ is given by:
\begin{equation*}
    \calI_h(\bz) = -\mathbb{E}_h\left[\nabla_\zeta^2 \log p_\calA(h|\zeta)|_{\zeta=\bz}\right],
\end{equation*}
where $p_\calA(h|\zeta)$ denotes the density of $h$ induced by the learning algorithm $\calA$ when $\bz = \zeta$. For example, if $\calA$ trains a linear regressor on $\calD_\text{train}$ with output perturbation~\cite{chaudhuri2011differentially}, then $p_\calA(h|\zeta)|_{\zeta = \bz}$ is the density function of $\calN(\bw^*, \sigma^2 I_d)$, with $\bw^*$ being the unique minimizer of the linear regression objective. Finally, FIL is defined as the spectral norm of the FIM: $\eta^2 = \| \calI_h(\bz) \|_2$.

\paragraph{Relationship to differential privacy.} There are close connections between FIL and the statistical divergences used to define DP. Fisher information measures the sensitivity of the density function $p_\calA(h|\zeta)|_{\zeta=\bz}$ with respect to the sample $\bz$. If FIL is zero, then the released model $h$ reveals no information about the sample $\bz$ since $\bz$ does not affect the (log) density of $h$. On the other hand, if FIL is large, then the (log) density of $h$ is very sensitive to change in $\bz$, hence revealing a lot of information about $\bz$.

It is noteworthy that DP is motivated by a similar reasoning. The divergence bound $D(\calA(\calD_\text{train})~||~\calA(\calD_\text{train}'))$ asserts that the sensitivity of $\calA$ to a single sample difference between $\calD_\text{train}$ and $\calD_\text{train}'$ is small, hence $h$ reveals very little information about any sample in $\calD_\text{train}$. In fact, it can be shown that Fisher information is the limit of chi-squared divergence~\cite{polyanskiy2020information}: For any $\bu \in \mathbb{R}^d$, $$\bu^\top \calI_h(\bz) \bu = \lim_{\Delta \rightarrow 0} \chi^2(\calA(\calD \cup \{\bz\})~||~\calA(\calD \cup \{\bz + \Delta \bu\})).$$ Since FIL is the spectral norm of $\calI_h(\bz)$, it upper bounds the chi-squared divergence between $\calA(\calD \cup \{\bz\})$ and $\calA(\calD \cup \{\bz'\})$ as $\bz' \rightarrow \bz$ from any direction. Crucially, this analysis is data-dependent and specific to each $\bz \in \calD_\text{train}$\footnote{This means that in practice, FIL should be kept secret to avoid unintended information leakage.}, while preserving the desirable properties of DP such as post-processing inequality~\cite{hannun2021measuring}, composition and subsampling (\autoref{sec:comp_and_sub}).

\paragraph{Cram\'{e}r-Rao bound.} FIL can be used to lower bound the MSE of DRAs via the Cram\'{e}r-Rao bound (CRB; \citet{kay1993fundamentals})---a well-known result for analyzing the efficiency of estimators (see \autoref{sec:crb_hcrb} for statement). We adapt the Cram\'{e}r-Rao bound to prove a similar MSE lower bound as in \autoref{thm:rdp_bound}. Proof is given in \autoref{sec:proofs}.

\begin{theorem}
\label{thm:fil_bound}
Assume the setup of \autoref{thm:rdp_bound}, and additionally that the log density function $\log p_\calA(h | \zeta)$ satisfies the regularity conditions in \autoref{thm:crb}. Then:
\begin{align*}
&\mathbb{E}[\|\hat{\bz}(h) - \bz\|_2^2/d] \geq \\
&\quad\quad \underbrace{\frac{\mathrm{Tr}(J_\mu(\bz) \calI_h(\bz)^{-1} J_\mu(\bz)^\top)}{d}}_\text{variance} + \underbrace{\frac{\|\mu(\bz) - \bz\|_2^2}{d}}_\text{squared bias}.
\end{align*}
In particular, if $\hat{\bz}(h)$ is unbiased then:
\begin{equation*}
\mathbb{E}[\|\hat{\bz}(h) - \bz\|_2^2/d] \geq d/\mathrm{Tr}(\calI_h(\bz)) \geq 1/\eta^2.
\end{equation*}
\end{theorem}

The bound in \autoref{thm:fil_bound} has a similar explicit bias-variance trade-off as that of \autoref{thm:rdp_bound}: The Jacobian $J_\mu(\bz)$ measures how sensitive the estimator $\hat{\bz}(h)$ is to $\bz$, which interacts with the FIM in the variance term. Notably, the bound for unbiased estimator decays \emph{quadratically} with respect to the privacy parameter $\eta$ as opposed to exponentially in \autoref{thm:rdp_bound}. We show in \autoref{sec:experiments} that this scaling also results in tighter MSE lower bounds in practice, yielding a better privacy-utility trade-off for the same private mechanism.

%% file: sections/fil_sgd.tex
\section{Private SGD with FIL Accounting}
\label{sec:fil_sgd}

Private SGD with Gaussian gradient perturbation~\cite{song2013stochastic, abadi2016deep} is a common technique for training DP models, especially neural networks.
In this section, we extend FIL accounting to the setting of private SGD by showing analogues of composition and subsampling bounds for FIL. This enables the use of \autoref{thm:fil_bound} to derive tighter per-sample estimates of vulnerability to data reconstruction attacks for private SGD learners.

\subsection{FIL Accounting for Composition and Subsampling}
\label{sec:comp_and_sub}

\paragraph{FIL for a single gradient step.} At time step $t \geq 1$, let $\calB_t \subseteq \calD_\text{train}$ be a batch of samples from $\calD_\text{train}$, and let $\bw_{t-1}$ be the model parameters before update. Denote by $\ell(\bz; \bw_{t-1})$ the loss of the model at a sample $\bz \in \calB_t$. Private SGD computes the update~\cite{abadi2016deep}:
\begin{align*}
    \bg_t(\bz) &\leftarrow \nabla_\bw \ell(\bz; \bw)|_{\bw = \bw_{t-1}} \quad \forall \bz \in \calB_t \\
    \tilde{\bg}_t(\bz) &\leftarrow \bg_t(\bz) / \max(1, \| \bg_t(\bz) \|_2 / C) \\
    \bar{\bg}_t &\leftarrow \frac{1}{|\calB_t|} \left( \sum_{\bz \in \calB_t} \tilde{\bg}_t(\bz) + \calN(\mathbf{0}, \sigma^2 C^2 \mathbf{I}) \right) \\
    \bw_t &\leftarrow \bw_{t-1} - \rho \bar{\bg}_t
\end{align*}
where $\mathbf{I}$ is the identity matrix, $C > 0$ is the per-sample clipping norm, $\sigma > 0$ is the noise multiplier, and $\rho > 0$ is the learning rate. Privacy is preserved using the Gaussian mechanism~\cite{dwork2014algorithmic} by adding $\calN(\mathbf{0}, \sigma^2 C^2 \mathbf{I})$ to the aggregate (clipped) gradient $\sum_{\bz \in \calB_t} \tilde{\bg}_t(\bz)$.
\citet{hannun2021measuring} showed that the Gaussian mechanism also satisfies FIL privacy, where the FIM is given by:
\begin{equation}
    \label{eq:per_step_fil}
    \calI_{\bar{\bg}_t}(\bz) = \left. \frac{1}{\sigma^2} \nabla_\zeta \tilde{\bg}_t(\zeta)^\top \nabla_\zeta \tilde{\bg}_t(\zeta) \right|_{\zeta = \bz}
\end{equation}
for any $\bz \in \calD_\text{train}$. In particular, if $\bz \notin \calB_t$ then $\calI_{\bar{\bg}_t}(\bz) = 0$. The quantity $\nabla_\zeta \tilde{\bg}_t(\zeta)$ is a second-order derivative of the clipped gradient $\tilde{\bg}_t(\zeta)$, which is computable using popular automatic differentiation packages such as PyTorch~\cite{paszke2019pytorch, functorch2021} and JAX~\cite{jax2018github}. We will discuss computational aspects of FIL for private SGD in \autoref{sec:computing}.

\paragraph{Composition of FIL across multiple gradient steps.} We first consider a simple case for composition where the batches are fixed. \autoref{thm:composition} shows that in order to compute the FIM for the final model $h$, it suffices to compute the per-step FIM $\calI_{\bar{\bg}_t}$ and take their sum.

\begin{theorem}
\label{thm:composition}
Let $\bw_0$ be the model's initial parameters, which is drawn independently of $\calD_\text{train}$. Let $T$ be the total number of iterations of SGD and let $\calB_1,\ldots,\calB_T$ be a fixed sequence of batches from $\calD_\text{train}$. Then:
\begin{equation*}
    \calI_h(\bz) \preceq \mathbb{E}_{\bw_0,\bar{\bg}_1,\ldots,\bar{\bg}_T}\left[ \sum_{t=1}^T \calI_{\bar{\bg}_t}(\bz | \bw_0,\bar{\bg}_1,\ldots,\bar{\bg}_{t-1})\right],
\end{equation*}
where $U \preceq V$ means that $V - U$ is positive semi-definite.
\end{theorem}

\autoref{thm:composition} has the following important practical implication: For each realization of $\bw_0,\bar{\bg}_1,\ldots,\bar{\bg}_T$ (\emph{i.e.}, a single training run), the realized FIM for that run can be computed by summing the per-step FIMs $\calI_{\bar{\bg}_t}$ conditioned on the \emph{realized} model parameter $\bw_{t-1}$ for $t=1,\ldots,T$. This gives an unbiased estimate of an upper bound for $\calI_h(\bz)$ via Monte-Carlo, and we can obtain a more accurate upper bound by repeating the training run multiple times and averaging.

\paragraph{Subsampling.} Privacy amplification by subsampling~\cite{kasiviswanathan2011can} is a powerful technique for reducing privacy leakage by randomizing the batches in private SGD: we draw each $\calB_t$ uniformly from the set of all $B$-subsets of $\calD_\text{train}$, where $B$ is the batch size. The following theorem shows that private SGD with FIL accounting also enjoys a subsampling amplification bound similar to DP~\cite{abadi2016deep} and RDP~\cite{wang2019subsampled, mironov2019r}; the proof is given in \autoref{sec:proofs}.

\begin{theorem}
\label{thm:subsampling}
Let $\hat{\bg}_t$ be the perturbed gradient at time step $t$ where the batch $\calB_t$ is drawn by sampling a subset of size $B$ from $\calD_\text{train}$ uniformly at random, and let $q = B / |\calD_\text{train}|$ be the sampling ratio. Then:
\begin{equation}
    \label{eq:convexity_bound}
    \calI_{\bar{\bg}_t}(\bz) \preceq \mathbb{E}_{\calB_t}[\calI_{\bar{\bg}_t}(\bz|\calB_t)].
\end{equation}
Furthermore, if the gradient perturbation mechanism is also $\epsilon$-DP, then:
\begin{equation}
    \label{eq:tight_subsampling_bound}
    \calI_{\bar{\bg}_t}(\bz) \preceq \frac{q}{q + (1-q) e^{-\epsilon}} \mathbb{E}_{\calB_t}[\calI_{\bar{\bg}_t}(\bz|\calB_t)].
\end{equation}
\end{theorem}

\paragraph{Accounting algorithm.} We can combine \autoref{thm:composition} and \autoref{thm:subsampling} to give the full FIL accounting equation for subsampled private SGD:
\begin{equation}
    \label{eq:fim_monte_carlo_subsample}
    \calI_h(\bz) \preceq \mathbb{E}_{\bw_0,\calB_1,\ldots,\calB_T,\bar{\bg}_1,\ldots,\bar{\bg}_T}\left[\kappa \sum_{t=1}^T \calI_{\bar{\bg}_t}(\bz)\right],
\end{equation}
where $\kappa$ is either $1$ or $q/(q+(1-q)e^{-\epsilon})$ depending on which bound in \autoref{thm:subsampling} is used. That is, we perform Monte-Carlo estimation of the FIM by randomizing the initial parameter vector $\bw_0$ and batches $\calB_1,\ldots,\calB_T$, and summing up the per-step FIMs.
%Each training run produces a single unbiased estimate of the upper bound in \autoref{eq:fim_monte_carlo_subsample}, and in practice we can repeat the training run multiple times and take the average to obtain a more accurate estimate.
Note that \autoref{eq:tight_subsampling_bound} in \autoref{thm:subsampling} depends on the DP privacy parameter $\epsilon$ of the gradient perturbation mechanism. It is well-known that the Gaussian mechanism satisfies $(\epsilon,\delta)$-DP where $\delta > 0$ and $\epsilon = 2\sqrt{2 \log(1.25 / \delta)} / \sigma$~\cite{dwork2014algorithmic}. Thus, when applying \autoref{thm:subsampling} to private SGD with Gaussian gradient perturbation, there is an arbitrarily small but non-zero probability $\delta$ that the tighter bound in \autoref{eq:tight_subsampling_bound} fails, and one must fall back to the simple bound in \autoref{eq:convexity_bound}. In practice, we set $\delta$ so that the total failure probability across all iterations $t=1,\ldots,T$ is at most $1/|\calD_\text{train}|$.

\begin{algorithm}[t]
\caption{FIL computation for private SGD.}
\label{alg:fil_sgd}
\begin{algorithmic}[1]
\STATE \textbf{Input}: Dataset $\calD_\text{train}$, learning rate $\rho > 0$, noise multiplier $\sigma > 0$, norm clip threshold $C > 0$, failure probability $\delta > 0$.
\STATE Initialize model parameters $\bw_0$ independently of $\calD_\text{train}$.
\STATE Initialize FIL accountant $\calI(\bz) = 0$ for all $\bz \in \calD_\text{train}$.
\STATE $\epsilon \gets 1.115 \cdot 2 \sqrt{2 \log(1.25 / \delta)} / \sigma, \; \kappa \gets \frac{q}{q + (1-q) e^{-\epsilon}}$
\FOR {$t \gets 1$ to $T$}
\STATE Sample batch $\calB_t$ uniformly at random from $\calD_\text{train}$ without replacement.
\FOR {$\bz \in \calB_t$}
\STATE $\bg_t(\bz) \gets \nabla_\bw \ell(\bz; \bw)|_{\bw = \bw_t}$
\STATE $\tilde{\bg}_t(\bz) \gets \bg_t(\bz) / (\mathrm{GELU}(\| \bg_t(\bz) \|_2 / C - 1) + 1)$
\STATE $\calI(\bz) \gets \calI(\bz) + \left. \frac{\kappa}{\sigma^2} \nabla_\zeta \tilde{\bg}_t(\zeta)^\top \nabla_\zeta \tilde{\bg}_t(\zeta) \right|_{\zeta = \bz}$
\ENDFOR
\STATE $\bar{\bg}_t \gets \frac{1}{|\calB_t|} \left( \sum_{\bz \in \calB_t} \tilde{\bg}_t(\bz) + \calN(\mathbf{0}, \sigma^2 C^2 \mathbf{I}) \right)$
\STATE $\bw_t \gets \bw_{t-1} - \rho \bar{\bg}_t$
\ENDFOR
\STATE \textbf{Return}: Fisher information upper bound $\{\calI(\bz)\}_{\bz \in \calD_\text{train}}$.
\end{algorithmic}
\end{algorithm}

\subsection{Computing FIL}
\label{sec:computing}

\paragraph{Handling non-differentiability.} The core quantity in FIL accounting is the per-step FIM of the gradient in \autoref{eq:per_step_fil}, which involves computing a second-order derivative $\nabla_\zeta \tilde{\bg}_t(\zeta)$ whose existence depends on the loss $\ell(\bz; \bw)$ being differentiable everywhere in both $\bz$ and $\bw$. This is not always the case since: 1. The network may have non-differentiable activation functions such as ReLU; and 2. The gradient norm clip operator requires computing $\max(1, \| \bg_t(\bz) \|_2 / C)$, which is also non-differentiable.

We address the first problem by replacing ReLU with the $\tanh$ activation function, which is smooth and has been recently found to be more suitable for private SGD training~\cite{papernot2020tempered}. The second problem can be addressed using the GELU function~\cite{hendrycks2016gaussian}, which is a smooth approximation to $\max(0, z)$. In particular, we replace the $\max(1, z)$ function with $\mathrm{GELU}(z-1) + 1$.

Algorithm \ref{alg:fil_sgd} summarizes the FIL computation with this modified norm clip operator in pseudo-code. We substitute hard gradient norm clipping using GELU in line 9. It can be verified that gradient norm clipping using GELU introduces a small multiplicative overhead in the clipping threshold:
$\| \tilde{\bg}_t(\bz) \|_2 \leq 1.115 C$ if $\tilde{\bg}_t(\bz) = \bg_t(\bz) / (\mathrm{GELU}(\| \bg_t(\bz) \|_2 / C - 1) + 1)$.

\paragraph{Improving computational efficiency.} Computation of the second-order derivative $\nabla_\zeta \tilde{\bg}_t(\zeta)$ can be done in JAX~\cite{jax2018github} using the \texttt{jacrev} operator. However, the dimensionality of the derivative $\nabla_\zeta \tilde{\bg}_t(\zeta)$ is $p \times d$, where $p$ is the number of model parameters and $d$ is the data dimensionality, which can be too costly to store in memory.
Fortunately, the bound for unbiased estimator in \autoref{thm:fil_bound} only requires computing either the trace or the spectral norm of $\calI_{\bar{\bg}_t}(\bz)$, which does not require instantiating the full second-order derivative $\nabla_\zeta \tilde{\bg}_t(\zeta)$. For instance,
\begin{equation}
    \label{eq:fim_trace}
    \mathrm{Tr}(\calI_{\bar{\bg}_t}(\bz)) = \sum_{i=1}^d \be_i^\top \calI_{\bar{\bg}_t}(\bz) \be_i = \sum_{i=1}^d \frac{\|\nabla_\zeta \tilde{\bg}_t(\zeta) \be_i |_{\zeta = \bz}\|_2^2}{\sigma^2},
\end{equation}
which can be computed using only Jacobian-vector products (\texttt{jvp} in JAX) without constructing the full Jacobian matrix. This can be done in Algorithm \ref{alg:fil_sgd} by modifying Line 10 accordingly. Furthermore, we can obtain an unbiased estimate of $\mathrm{Tr}(\calI_{\bar{\bg}_t}(\bz))$ by sampling the coordinates $i=1,\ldots,d$ in \autoref{eq:fim_trace} stochastically. Doing so gives a Monte-Carlo estimate of $\mathrm{Tr}(\calI_h(\bz))$ using \autoref{eq:fim_monte_carlo_subsample} since trace is a linear operator. Similarly, we can compute the spectral norm using JVP via power iteration.

%% file: sections/experiments.tex
\section{Experiments}
\label{sec:experiments}

We evaluate our MSE lower bounds in \autoref{thm:rdp_bound} and \autoref{thm:fil_bound} for unbiased estimators and show that RDP and FIL both provide meaningful semantic guarantees against DRAs.
In addition, we evaluate the informed adversary attack~\cite{balle2022reconstructing} against privately trained models and show that a sample's vulnerability to this reconstruction attack is closely captured by the FIL lower bound. Code to reproduce our results is available at \url{https://github.com/facebookresearch/bounding_data_reconstruction}.

\subsection{Linear Logistic Regression}
\label{sec:logistic}

We first consider linear logistic regression for binary MNIST~\cite{lecun1998gradient} classification of digits $0$ vs. $1$. The training set contains $n=12,665$ samples. Each sample $\bz = (\bx, y)$ consists of an input image $\bx \in [0,1]^{784}$ and a label $y \in \{0,1\}$. Since the value of $y$ is discrete, we treat the label as public and only seek to prevent reconstruction of the image $\bx$.

\paragraph{Privacy accounting.} The linear logistic regressor is trained privately using output perturbation~\cite{chaudhuri2011differentially} by adding Gaussian noise $\mathcal{N}(0, \sigma^2 \mathbf{I})$ to the non-private model. For a given L2 regularization parameter $\lambda > 0$ and noise parameter $\sigma > 0$, it can be shown that output perturbation satisfies $(2,\epsilon)$-RDP with $\epsilon = 4/(n \lambda \sigma)^2$. For FIL accounting, we follow \citet{hannun2021measuring} and compute the full Fisher information matrix $\calI_h(\bz)$, then take the average diagonal value $\bar{\eta}^2 := \mathrm{Tr}(\calI_h(\bz))/d$ in order to apply \autoref{thm:fil_bound}. The final estimate is computed as an average across 10 runs. We refer to this quantity as the \emph{diagonal Fisher information loss} (dFIL).

\paragraph{Result.} We train the model with $\lambda=10^{-2}$ and $\sigma=10^{-2}$, achieving a near-perfect test accuracy of $99.95\%$ and $(2,\epsilon)$-RDP with $\epsilon = 2.49$. \autoref{fig:mnist_linear_hist} shows the RDP lower bound in \autoref{thm:rdp_bound} and the histogram of per-sample dFIL lower bounds in \autoref{thm:fil_bound}. Since the data space is $[0,1]^{784}$, we have that $\mathrm{diam}_i(\calZ) = 1$ for all $i$, so the RDP bound reduces to $\texttt{MSE} \geq 1/(4(e^\epsilon - 1))$, while the dFIL bound is $\texttt{MSE} \geq 1/\bar{\eta}^2$. The plot shows that the RDP bound is $\approx 0.02$, while all the per-sample dFIL bounds are $>1$. Since $\texttt{MSE} \leq 1$ can be achieved by simply guessing any value within $[0,1]^{784}$, we regard the vertical line of $\texttt{MSE} = 1$ as perfect privacy. Hence the dFIL predicts that all training samples are safe from reconstruction attacks.
Moreover, there is an extremely wide range of values for the per-sample dFIL bounds. We show in the following experiment that these values are highly indicative of how susceptible the sample is to an actual data reconstruction attack.

\begin{figure}[t!]
\centering
\includegraphics[width=\linewidth]{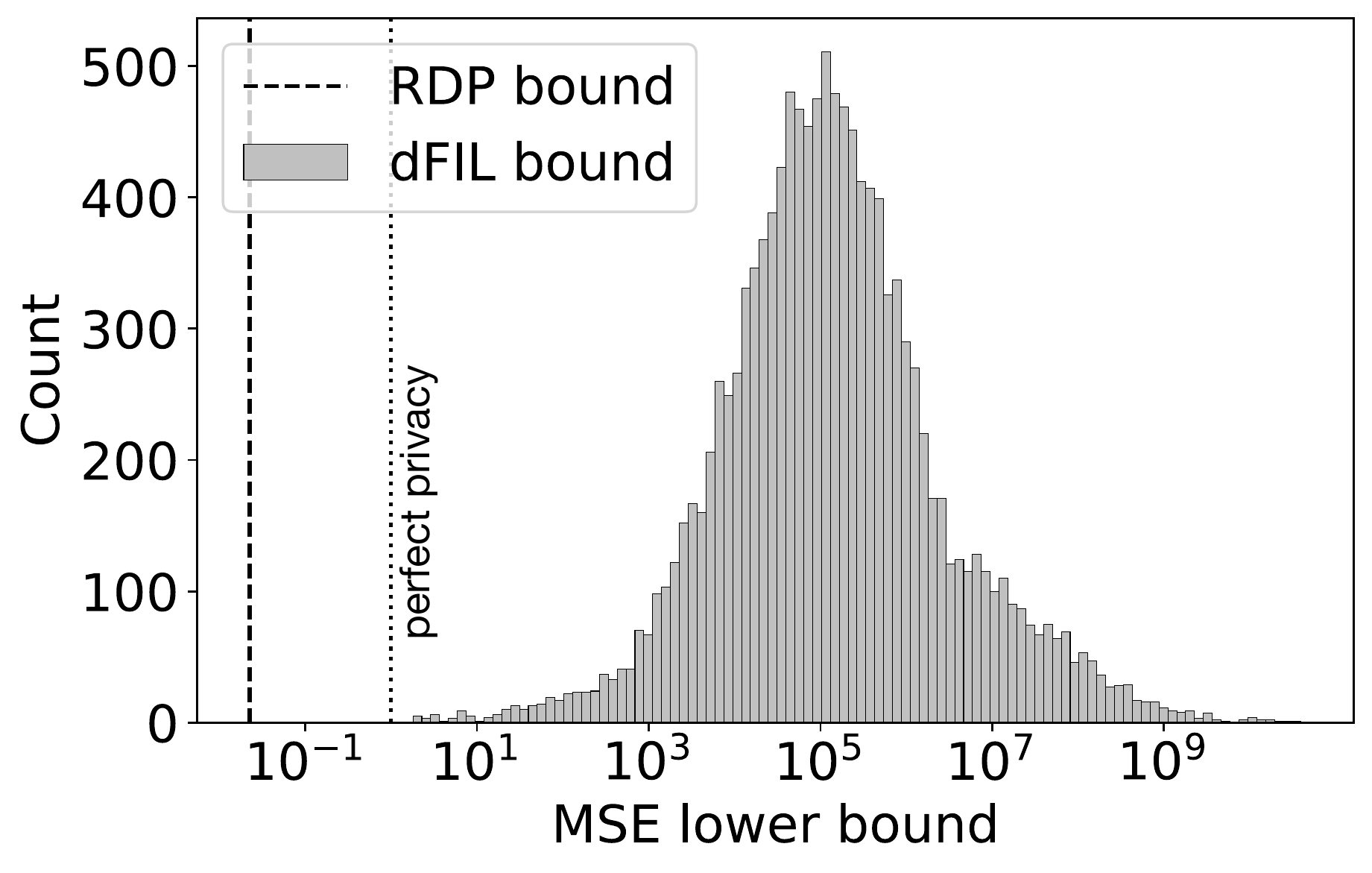}
\caption{Plot showing the RDP lower bound and histogram of the per-sample FIL lower bound for the MNIST 0 vs. 1 classifier. The vertical line at $\texttt{MSE} = 1$ represents the perfect privacy threshold, which is the MSE attainable by a random guessing adversary.}
\label{fig:mnist_linear_hist}
\end{figure}

\begin{figure}[t!]
\centering
\includegraphics[width=\linewidth]{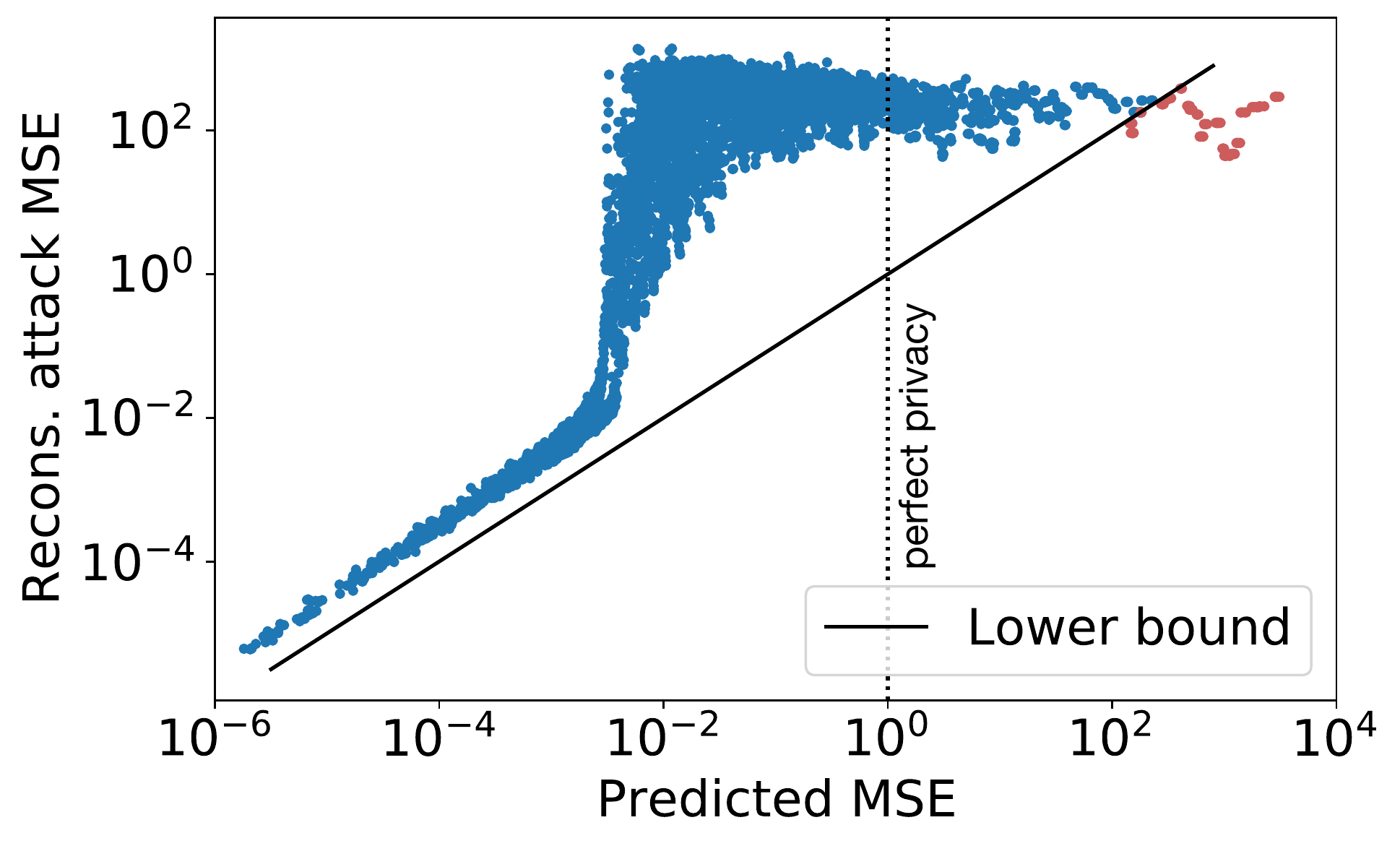}
\caption{Scatter plot of the MSE lower bound from FIL (x-axis) and the MSE realized by the GLM attack (y-axis; \citet{balle2022reconstructing}). The MSE lower bound predicted by \autoref{thm:fil_bound} is highly indicative of the sample's vulnerability to the GLM attack.}
\label{fig:recons_mse}
\end{figure}

\begin{figure*}[t!]
\centering
\includegraphics[width=\linewidth]{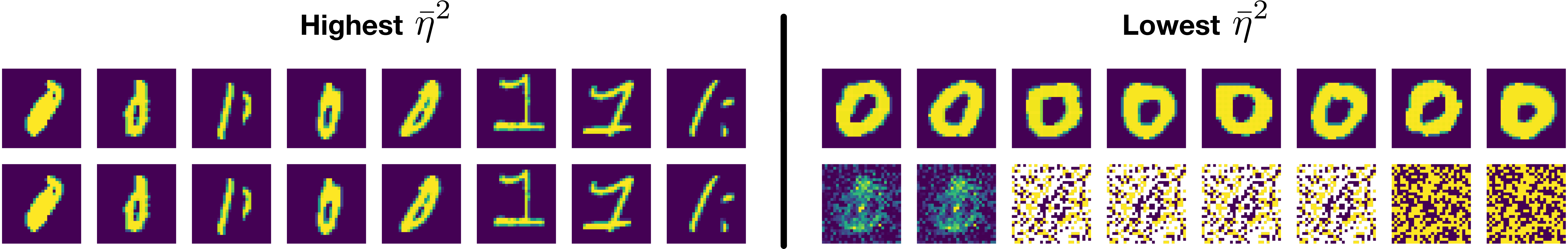}
\caption{Training samples (\textbf{top row}) and their reconstructions (\textbf{bottom row}) by the GLM attack. Samples are sorted in decreasing order of the dFIL $\bar{\eta}^2$. Samples with high dFIL can be reconstructed perfectly, while ones with low dFIL are protected against the GLM attack.}
\label{fig:recons_samples}
\end{figure*}

\begin{figure*}[t!]
\centering
\begin{subfigure}{.49\textwidth}
  \centering
  \includegraphics[width=\linewidth]{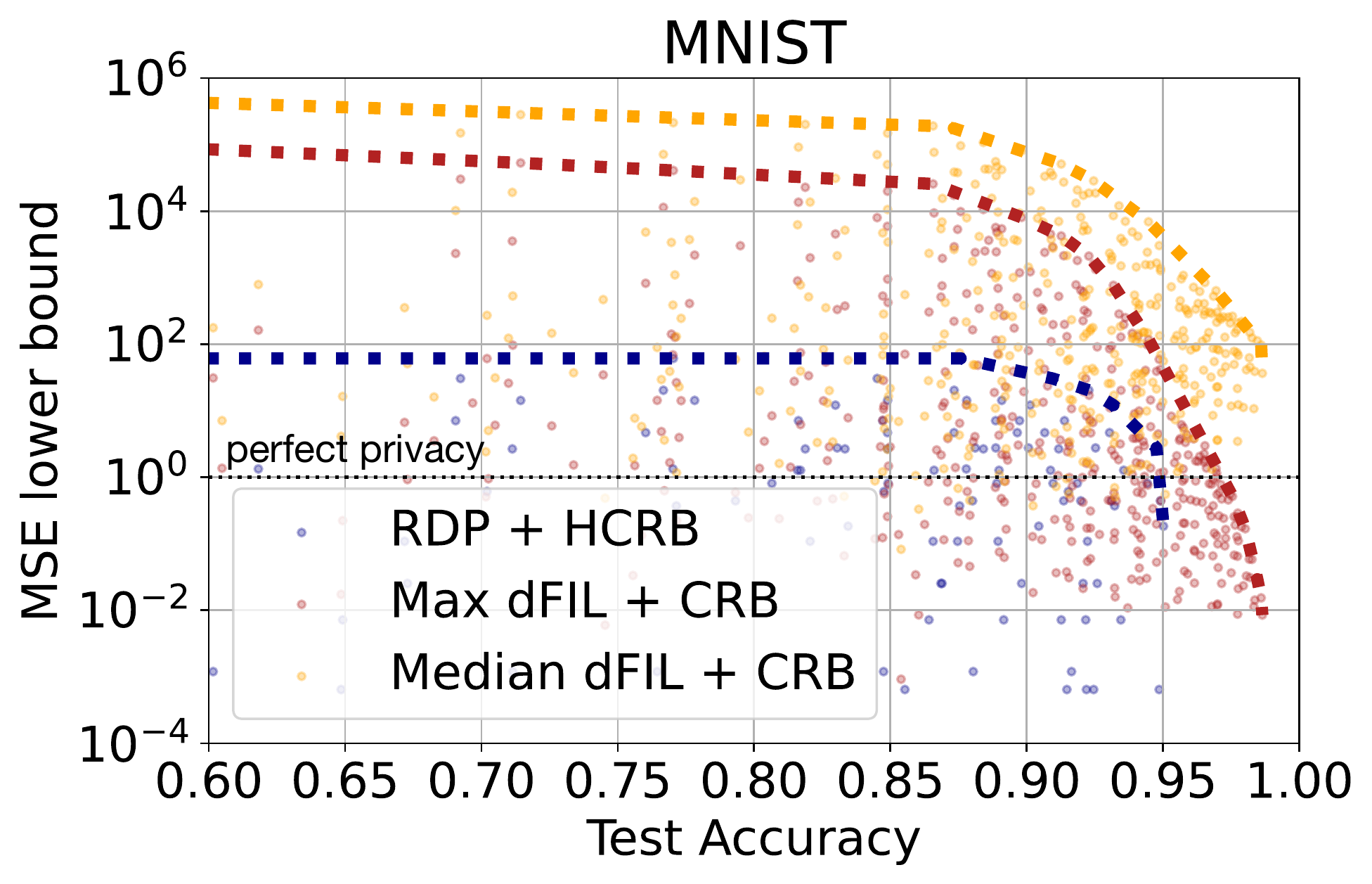}
\end{subfigure}
\hfill
\begin{subfigure}{.49\textwidth}
  \centering
  \includegraphics[width=\linewidth]{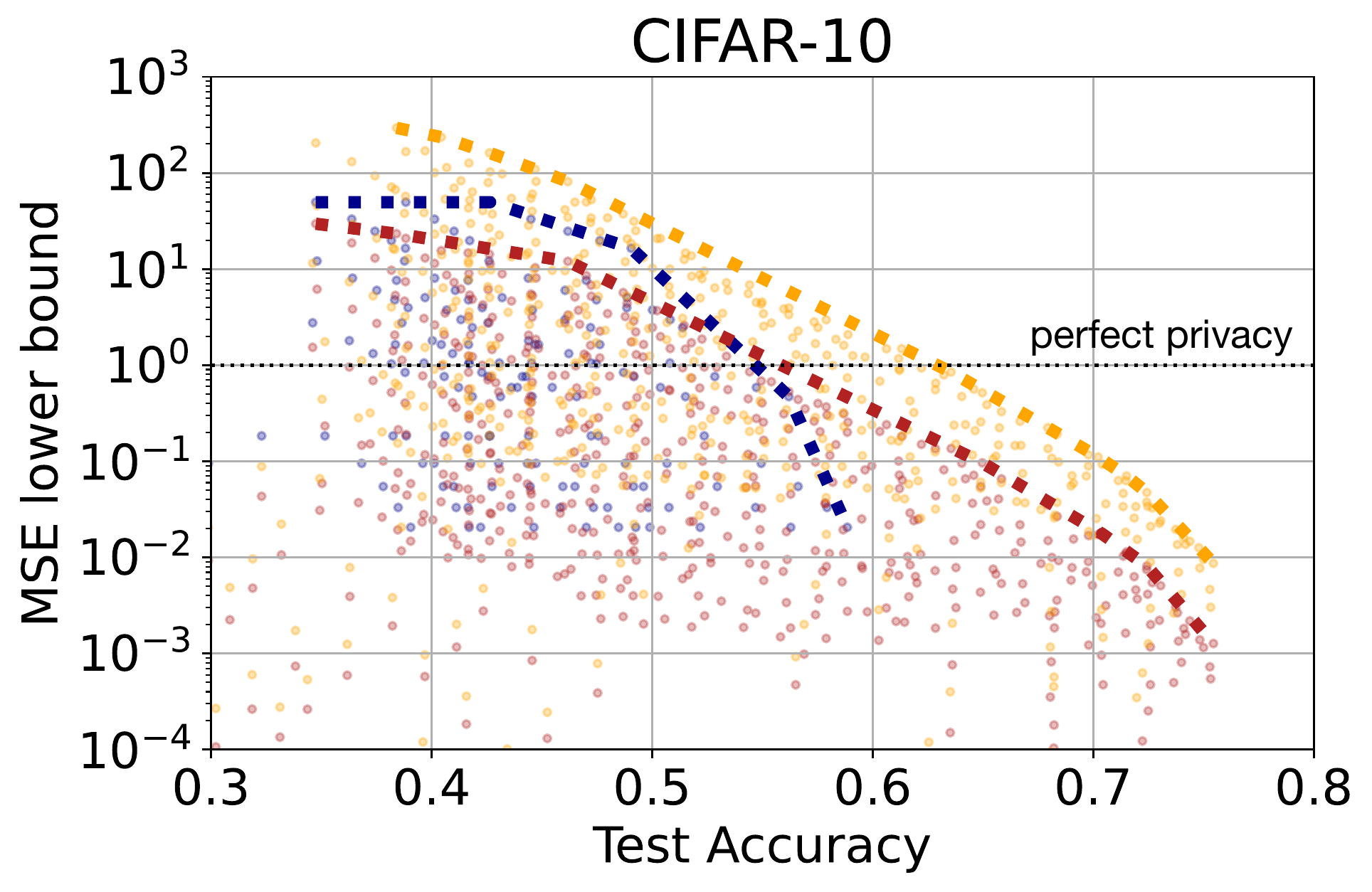}
\end{subfigure}
\caption{Comparison of MSE lower bounds from RDP and FIL. Dashed line shows the optimal privacy-utility trade-off across all searched hyperparameters. The maximum dFIL across the dataset gives a better MSE lower bound compared to the RDP bound in most settings.}
\label{fig:mse_comparison}
\end{figure*}

\subsection{Lower Bounding Reconstruction Attack}
\label{sec:reconstruction}

\citet{balle2022reconstructing} proposed a strong data reconstruction attack against generalized linear models (GLMs). We strengthen the attack by providing the label $y$ of the target sample $\bz = (\bx, y)$ to the adversary in addition to all other samples in $\calD_\text{train}$.
%The attack relies on the threat model defined in \autoref{fig:dra_attack}, where the adversary has full knowledge of all samples in $\calD_\text{train}$ except for the target sample $\bz$. In addition, we further provide the label $y$ of the target sample $\bz = (\bx, y)$ to the adversary.
To evaluate the GLM attack, we train a private model using output perturbation with $\lambda=10^{-2}$ and $\sigma=10^{-5}$, and apply the GLM attack to reconstruct each sample in the training set. We repeat this process $10,000$ times and compute the expected MSE across the trials. The noise parameter $\sigma$ is intentionally set to be very small to enable data reconstruction on some vulnerable samples. Under this setting, the model is $(2,\epsilon)$-RDP with $\epsilon = 2.5 \times 10^6$, which is too large to provide any meaningful privacy guarantee.

\paragraph{Result.} \autoref{fig:recons_mse} shows the scatter plot of MSE lower bounds predicted by the dFIL bound (x-axis) vs. the realized expected MSE of the GLM attack (y-axis). The solid line shows the cut-off for the dFIL lower bound, hence all points should be above the solid line if the dFIL bound holds. We see that this is indeed the case for the majority of samples: lower predicted MSE corresponds to lower realized MSE, and there is a close correlation between the two values especially at the lower end.
We observe that some samples (highlighted in red) violate the dFIL lower bound. One explanation is that the GLM attack incurs a high bias when the sample is hard to reconstruct, hence the unbiased bound in \autoref{thm:fil_bound} fails to hold for those samples. Nevertheless, we see that for all samples with $\texttt{MSE} \leq 1$ (to the left of the perfect privacy line), the dFIL bound does provide a meaningful semantic guarantee against the GLM attack.

\paragraph{Reconstructed samples.} \autoref{fig:recons_samples} shows selected training samples (top row) and their reconstructions (bottom row). Samples are sorted in decreasing order of dFIL $\bar{\eta}^2$ and only the top- and bottom-8 are shown. For samples with the highest dFIL (\emph{i.e.}, lowest MSE bounds), the GLM attack successfully reconstructs the sample, while the attack fails for samples with the lowest dFIL. %This evaluation shows that the predicted MSE bounds translate to visual quality of the reconstructed samples as well.

\begin{figure*}[t!]
\centering
\includegraphics[width=\linewidth]{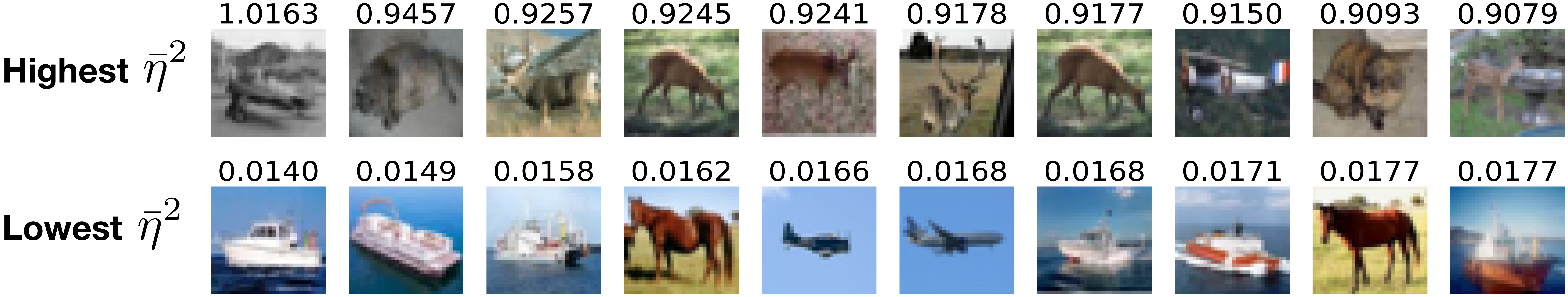}
\caption{CIFAR-10 training samples with the highest and lowest dFIL values $\bar{\eta}$.}
\label{fig:cifar10_samples}
\end{figure*}

\subsection{Neural Networks}
\label{sec:neural_network}

Finally, we compare MSE lower bounds for RDP and FIL accounting for the private SGD learner. We train two distinct convolutional networks\footnote{We adapt the networks used in \citet{papernot2020tempered}; see \autoref{sec:additional_details} for details.} on the full 10-digit MNIST~\cite{lecun1998gradient} dataset and the CIFAR-10~\cite{krizhevsky2009learning} dataset. The learner has several hyperparameters, and we exhaustively evaluate on all hyperparameter settings via grid search; see \autoref{sec:additional_details} for details. Similar to the experiment in \autoref{sec:logistic}, we treat the label as public and compute MSE bounds for reconstructing the input $\bx$.

\paragraph{Privacy accounting.} For RDP accounting, we apply the subsampling bound in \citet{mironov2019r}. For FIL accounting we use Algorithm \ref{alg:fil_sgd}, and estimate $\mathrm{Tr}(\calI_{\tilde{\bg}_t}(\bz))$ by sampling $50$ coordinates randomly every iteration (see \autoref{eq:fim_trace}). Failure probability for the subsampling bound in \autoref{eq:tight_subsampling_bound} is set to $\delta < 10^{-5}$. Each training run is repeated $10$ times to give a Monte-Carlo estimate for dFIL.

\paragraph{Result.} \autoref{fig:mse_comparison} shows the MSE lower bounds from RDP and dFIL on MNIST (left) and CIFAR-10 (right). Each point in the scatter plot corresponds to a single hyperparameter configuration, where we show the test accuracy on the x-axis and the MSE lower bound on the y-axis. In addition, we show the Pareto frontier using the dashed line, which indicates the optimal privacy-utility trade-off found by the grid search. In both plots, the RDP bound (shown in blue) gives a meaningful MSE lower bound, where the model can attain a reasonable accuracy ($95\%$ for MNIST and $55\%$ for CIFAR-10) before crossing the perfect privacy threshold.

The dFIL bound paints a more optimistic picture: For the same private mechanism, the maximum dFIL across the training set (shown in red) combined with \autoref{thm:fil_bound} gives an MSE lower bound that is orders of magnitude higher than the RDP bound at higher accuracies. On MNIST, the model can attain $97\%$ test accuracy before crossing the perfect privacy threshold. On CIFAR-10, although the dFIL bound crosses the perfect privacy threshold at approximately the same accuracy as the RDP bound, the bound deteriorates much more gradually, giving a non-negligible privacy guarantee of $\texttt{MSE} \geq 0.1$ at test accuracy $64\%$. Moreover, the median dFIL (shown in orange) indicates that even at high levels of accuracy, the median MSE lower bound across the dataset remains relatively high, hence most training samples are still safe from reconstruction attacks.

\paragraph{CIFAR-10 samples.} \autoref{fig:cifar10_samples} shows CIFAR-10 training samples with the highest and lowest privacy leakage according to dFIL ($\bar{\eta}$; shown above each image) for a ConvNet model trained privately with $T=5000, \sigma=0.5, \rho=0.1$ and $C=1$. Qualitatively, samples with low privacy leakage (bottom row) are typical images for their class and are easy to recognize, while samples with high privacy leakage (top row) are difficult to classify correctly even for humans.

%% file: sections/discussion.tex
\section{Discussion}
\label{sec:discussion}

We presented a formal framework for analyzing data reconstruction attacks, and proved two novel lower bounds on the MSE of reconstructions for private learners using RDP and FIL accounting.
Our work also extended FIL accounting to private SGD, and we showed that the resulting MSE lower bounds drastically improve upon those derived from RDP.
We hope that future research can build upon our work to develop more comprehensive analytical tools for evaluating the privacy risks of learning algorithms.

Concurrent work by \citet{balle2022reconstructing} offered a Bayesian approach to bounding reconstruction attacks. Their formulation lower bounds the reconstruction error of an adversary in terms of the error of an adversary with only access to the data distribution prior and the DP parameter $\epsilon$. In contrast, our bounds characterize the prior of an adversary in terms of sensitivity of their estimate to the training data, with a priorless adversary being unbiased and hence the most sensitive. Interestingly, the Bayesian extension~\citep{van2004detection} of the Cram\'{e}r-Rao bound used in our result offers a similar interpretation as \citet{balle2022reconstructing}, and we hope to unite these two interpretations in future work.

\paragraph{Limitations.} Our work presents several opportunities for further improvement.

1. The RDP bound only applies natively for order $\alpha=2$. To extend it to general order $\alpha$, one promising direction is to use minimax bounds~\cite{rigollet2015high} to establish a relationship between parameter estimation and hypothesis testing, which enables the use of general DP accountants to derive MSE lower bounds for DRAs.

2. Computing the FIM requires evaluating a second-order derivative, which is much more expensive (in terms of both compute and memory) to derive than simpler quantities such as R\'{e}nyi divergence. Improvements in this aspect can enable the use of the FIL accountant in larger models.

3. We empirically evaluated both the RDP and FIL lower bounds only for unbiased adversaries. In practice, data reconstruction attacks can leverage informative priors such as the smoothness prior for images, and hence are unlikely to be truly unbiased. Further investigation into MSE lower bounds for biased estimators can enable more robust semantic guarantees against data reconstruction attacks.

%% file: sections/appendix.tex
\section{Cram\'{e}r-Rao and Hammersley-Chapman-Robbins Bounds}
\label{sec:crb_hcrb}

Below, we state the Cram\'{e}r-Rao Bound (CRB) and the Hammersley-Chapman-Robbins Bound (HCRB)---two cornerstone results in statistics that we leverage for proving our main results.

\setcounter{theorem}{0}
\renewcommand{\thetheorem}{\Alph{section}.\arabic{theorem}}

\begin{theorem}[Hammersley-Chapman-Robbins Bound]
\label{thm:hcrb}
Let $\theta \in \Theta \subseteq \mathbb{R}^d$ be a parameter vector and let $U$ be a random variable whose density function $p(\bu; \theta)$ is parameterized by $\theta$ and is positive for all $\bu \in \mathbb{R}^p$ and $\theta \in \Theta$. Let $\hat{\theta}(U)$ be an estimator of $\theta$ whose expectation is $\mu(\theta) := \mathbb{E}_{U \sim p(\bu; \theta)}[\hat{\theta}(U)]$. Then for any $i \in \{1,\ldots,d\}$ and $\Delta \in \mathbb{R}$, we have that:
\begin{align*}
    \var\left(\hat{\theta}(U)_i\right) \geq \frac{\quad (\mu(\theta + \Delta \be_i)_i - \mu(\theta)_i)^2}{\chi^2(p(\bu; \theta + \Delta \be_i)~||~p(\bu; \theta))},
\end{align*}
where $\be_i$ is the standard basis vector with $i$th coordinate equal to 1, and $\chi^2(P~||~Q) = \mathbb{E}_Q[(P/Q - 1)^2]$ is the chi-squared divergence between $P$ and $Q$.
\end{theorem}
\begin{proof}
First note that
\begin{align*}
    \mu(\theta + \Delta \be_i) - \mu(\theta) &= \mathbb{E}_{U \sim p(\bu; \theta + \Delta \be_i)}[\hat{\theta}(U) - \mu(\theta)] - \mathbb{E}_{U \sim p(\bu; \theta)}[\hat{\theta}(U) - \mu(\theta)] \\
    &= \mathbb{E}_{U \sim p(\bu; \theta)}\left[ (\hat{\theta}(U) - \mu(\theta)) \frac{p(\bu; \theta + \Delta \be_i) - p(\bu; \theta)}{p(\bu; \theta)} \right].
\end{align*}
Squaring and applying Cauchy-Schwarz gives
\begin{align*}
    (\mu(\theta + \Delta \be_i)_i - \mu(\theta)_i)^2 &\leq \mathbb{E}_{U \sim p(\bu; \theta)}\left[ (\hat{\theta}(U)_i - \mu(\theta)_i)^2 \right] \mathbb{E}_{U \sim p(\bu; \theta)}\left[ \left( \frac{p(\bu; \theta + \Delta \be_i) - p(\bu; \theta)}{p(\bu; \theta)} \right)^2 \right] \\
    &= \var(\hat{\theta}(U)_i) \chi^2(p(\bu; \theta + \Delta \be_i)~||~p(\bu; \theta)),
\end{align*}
as desired.
\end{proof}

\begin{theorem}[Cram\'{e}r-Rao Bound]
\label{thm:crb}
Assume the setup of \autoref{thm:hcrb}, and additionally that the log density function $\log p(\bu;\theta)$ is twice differentiable and satisfies the following regularity condition: $\mathbb{E}[\partial \log p(\bu; \theta) / \partial \theta] = 0$ for all $\theta$. Let $\calI_U(\theta)$ be the Fisher information matrix of $U$ for the parameter vector $\theta$. Then the estimator $\hat{\theta}(U)$ satisfies: $$\var\left( \hat{\theta}(U) \right) \succeq J_\mu(\theta) \calI_U(\theta)^{-1} J_\mu(\theta)^\top,$$ where $J_\mu(\theta)$ is the Jacobian of $\mu$ with respect to $\theta$.
\end{theorem}

\section{Proofs}
\label{sec:proofs}

We present proofs of theoretical results from the main text.

\setcounter{theorem}{0}
\renewcommand{\thetheorem}{\arabic{theorem}}

\begin{theorem}
Let $\bz \in \calZ \subseteq \mathbb{R}^d$ be a sample in the data space $\calZ$, and let $\texttt{Att}$ be a reconstruction attack that outputs $\hat{\bz}(h)$ upon observing the trained model $h \leftarrow \calA(\calD_\text{train})$, with expectation $\mu(\bz) = \mathbb{E}_{\calA(\calD_\text{train})}[\hat{\bz}(h)]$. If $\calA$ is a $(2,\epsilon)$-RDP learning algorithm then:
\begin{equation*}
    \mathbb{E}\left[\|\hat{\bz}(h) - \bz\|_2^2 / d\right] \geq \underbrace{\frac{\sum_{i=1}^d \gamma_i^2 \mathrm{diam}_i(\calZ)^2/4d}{e^\epsilon - 1}}_\text{variance} + \underbrace{\frac{\|\mu(\bz) - \bz\|_2^2}{d}}_\text{squared bias},
\end{equation*}
where $\gamma_i = \inf_{\bz \in \calZ} |\partial \mu(\bz)_i / \partial \bz_i|$ and
$$\mathrm{diam}_i(\calZ) = \sup_{\bz, \bz' \in \calZ : \bz_j = \bz'_j \forall j \neq i} |\bz_i - \bz'_i|$$
is the diameter of $\calZ$ in the $i$-th dimension. In particular, if $\hat{\bz}(h)$ is unbiased then:
\begin{equation*}
    \mathbb{E}[\|\hat{\bz}(h) - \bz\|_2^2 / d] \geq \frac{\sum_{i=1}^d \mathrm{diam}_i(\calZ)^2/4d}{e^\epsilon - 1}.
\end{equation*}
\end{theorem}
\begin{proof}
Let $p(h; \bz')$ be the density of $h \leftarrow \calA(\calD_\text{train})$ when $\calD_\text{train} = \calD \cup \{\bz'\}$. We first invoke the well-known identity $D_2(P~||~Q) = \log(1 + \chi^2(P~||~Q))$, hence if $\calA$ is $(2,\epsilon)$-RDP then $\chi^2(p(h; \bz + \Delta \be_i)~||~p(h; \bz)) \leq e^\epsilon - 1$.
For each $i=1,\ldots,d$, we can apply bias-variance decomposition to get $(\hat{\bz}(h)_i - \bz_i)^2 = \var(\hat{\bz}(h)_i) + (\mu(\bz)_i - \bz_i)^2$.
Applying \autoref{thm:hcrb} to the variance term gives:
\begin{align*}
    \var(\hat{\bz}(h)_i) &\geq (\mu(\bz + \Delta \be_i)_i - \mu(\bz)_i)^2 / \chi^2(p(h; \bz + \Delta \be_i)~||~p(h; \bz)) \\
    &\geq (\mu(\bz + \Delta \be_i)_i - \mu(\bz)_i)^2 / (e^\epsilon - 1) \\
    &\geq \gamma_i^2 \Delta^2 / (e^\epsilon - 1),
\end{align*}
where the last inequality follows from the mean value theorem. Since this holds for any $\Delta$, we can maximize over $\{\Delta \in \mathbb{R} : \bz + \Delta \be_i \in \calZ\}$, which gives $\var(\hat{\bz}(h)_i) \geq \gamma_i^2 \mathrm{diam}_i(\calZ)^2/4(e^\epsilon - 1)$. Summing over $i=1,\ldots,d$ gives the desired bound.
If $\hat{\bz}(h)$ is unbiased then $\mu(\bz) = \bz$ and $\gamma_i = 1$ for all $i$, thus $\mathbb{E}[\|\hat{\bz}(h) - \bz\|_2^2 / d] \geq \frac{\sum_{i=1}^d \mathrm{diam}_i(\calZ)^2/4d}{e^\epsilon - 1}$.
\end{proof}

\begin{theorem}
Assume the setup of \autoref{thm:rdp_bound}, and additionally that the log density function $\log p_\calA(h | \zeta)$ satisfies the conditions in \autoref{thm:crb}. Then:
\begin{equation*}
\mathbb{E}[\|\hat{\bz}(h) - \bz\|_2^2/d] \geq \underbrace{\frac{\mathrm{Tr}(J_\mu(\bz) \calI_h(\bz)^{-1} J_\mu(\bz)^\top)}{d}}_\text{variance} + \underbrace{\frac{\|\mu(\bz) - \bz\|_2^2}{d}}_\text{squared bias}.
\end{equation*}
In particular, if $\hat{\bz}(h)$ is unbiased then:
\begin{equation*}
\mathbb{E}[\|\hat{\bz}(h) - \bz\|_2^2/d] \geq d/\mathrm{Tr}(\calI_h(\bz)) \geq 1/\eta^2.
\end{equation*}
\end{theorem}
\begin{proof}
The general bound for biased estimators follows directly from \autoref{thm:crb} and bias-variance decomposition of MSE. For the unbiased estimator bound, note that the Jacobian $J_\mu(\bz) = I_d$, so $$\mathbb{E}[\|\hat{\bz}(h) - \bz\|_2^2/d] \geq \mathrm{Tr}(\calI_h(\bz)^{-1})/d \geq d^2 \mathrm{Tr}(\calI_h(\bz))^{-1} / d = d / \mathrm{Tr}(\calI_h(\bz)),$$ where the second inequality follows from Cauchy-Schwarz. Finally, $\mathrm{Tr}(\calI_h(\bz)) = \sum_{i=1}^d \be_i^\top \calI_h(\bz) \be_i \leq \sum_{i=1}^d \eta^2 \| \be_i \|_2^2 = d \eta^2$, and the result follows.
\end{proof}

\begin{theorem}
Let $\bw_0$ be the model's initial parameters, which is drawn independently of $\calD_\text{train}$. Let $T$ be the total number of iterations of SGD and let $\calB_1,\ldots,\calB_T$ be a fixed sequence of batches from $\calD_\text{train}$. Then:
\begin{equation*}
    \calI_{\bw_0,\bar{\bg}_1,\ldots,\bar{\bg}_T}(\bz) \preceq \mathbb{E}_{\bw_0,\bar{\bg}_1,\ldots,\bar{\bg}_T}\left[ \sum_{t=1}^T \calI_{\bar{\bg}_t}(\bz | \bw_0,\bar{\bg}_1,\ldots,\bar{\bg}_{t-1})\right],
\end{equation*}
where $U \preceq V$ means that $V - U$ is positive semi-definite.
\end{theorem}
\begin{proof}
First note that the final model $h \leftarrow \calA(\calD_\text{train})$ is a deterministic function of only the initial parameters $\bw_0$ and the observed gradients $\bar{\bg}_1,\ldots,\bar{\bg}_T$ without any other access to $\bz$, hence by the post-processing inequality for Fisher information~\cite{zamir1998proof}, we get $\calI_{h}(\bz) \preceq \calI_{\bw_0,\bar{\bg}_1,\ldots,\bar{\bg}_T}(\bz)$.
To bound $\calI_{\bw_0,\bar{\bg}_1,\ldots,\bar{\bg}_T}(\bz)$ for any $\bz \in \calD_\text{train}$, we apply the chain rule for Fisher information~\cite{zamir1998proof}:
\begin{align*}
    \calI_{\bw_0,\bar{\bg}_1,\ldots,\bar{\bg}_T}(\bz) &= \calI_{\bw_0}(\bz) + \mathbb{E}_{\bw_0,\bar{\bg}_1,\ldots,\bar{\bg}_T}\left[ \sum_{t=1}^T \calI_{\bar{\bg}_t}(\bz | \bw_0,\bar{\bg}_1,\ldots,\bar{\bg}_{t-1})\right] \\
    &= \mathbb{E}_{\bw_0,\bar{\bg}_1,\ldots,\bar{\bg}_T}\left[ \sum_{t=1}^T \calI_{\bar{\bg}_t}(\bz | \bw_0,\bar{\bg}_1,\ldots,\bar{\bg}_{t-1})\right],
\end{align*}
where $\calI_{\bw_0}(\bz) = 0$ since $\bw_0$ is independent of the training data. The quantity $\calI_{\bar{\bg}_t}(\bz|\bw_0,\bar{\bg}_1,\ldots,\bar{\bg}_{t-1})$ represents the conditional Fisher information, which depends on the current model parameter $\bw_{t-1}$ through $\bw_0,\bar{\bg}_1,\ldots,\bar{\bg}_{t-1}$.
\end{proof}

\begin{theorem}
Let $\hat{\bg}_t$ be the perturbed gradient at time step $t$ where the batch $\calB_t$ is drawn by sampling a subset of size $B$ from $\calD_\text{train}$ uniformly randomly, and let $q = B / |\calD_\text{train}|$ be the sampling ratio. Then:
\begin{equation*}
    \calI_{\bar{\bg}_t}(\bz) \preceq \mathbb{E}_{\calB_t}[\calI_{\bar{\bg}_t}(\bz|\calB_t)].
\end{equation*}
Furthermore, if the gradient perturbation mechanism is also $\epsilon$-DP, then:
\begin{equation*}
    \calI_{\bar{\bg}_t}(\bz) \preceq \frac{q}{q + (1-q) e^{-\epsilon}} \mathbb{E}_{\calB_t}[\calI_{\bar{\bg}_t}(\bz|\calB_t)].
\end{equation*}
\end{theorem}
\begin{proof}
The first bound follows from convexity of Fisher information. Let $\calB_t^1$ and $\calB_t^2$ be two batches and let $p_1,p_2$ be the density functions of the perturbed batch gradient $\bar{\bg}_t$ corresponding to the two batches. For any $\lambda \in (0,1)$, let $\calI_{\bar{\bg}_t}(\bz)$ be the FIM for the mixture distribution with $\mathbb{P}(\calB_t^1) = \lambda$ and $\mathbb{P}(\calB_t^2) = 1 - \lambda$. We will show that:
\begin{equation}
\label{eq:mixture_fi}
\calI_{\bar{\bg}_t}(\bz) \preceq \lambda \calI_{\bar{\bg}_t}(\bz | \calB_t^1) + (1-\lambda) \calI_{\bar{\bg}_t}(\bz | \calB_t^2).
\end{equation}
For any $\bu \in \mathbb{R}^p$, observe that:
\begin{align}
\bu^\top \calI_{\bar{\bg}_t}(\bz | \calB_t^1) \bu &= \int_{\bar{\bg}_t} \bu^\top \left[ \left. \nabla_\zeta \log p_1(\bar{\bg}_t|\zeta)
\nabla_\zeta \log p_1(\bar{\bg}_t|\zeta)^\top \right\vert_{\zeta = \bz} \right] \bu \: p_1(\bar{\bg}_t|\zeta) \; d\bar{\bg}_t \nonumber \\
&= \int_{\bar{\bg}_t} \bu^\top \left[ \left. \nabla_\zeta p_1(\bar{\bg}_t|\zeta)
\nabla_\zeta p_1(\bar{\bg}_t|\zeta)^\top \right\vert_{\zeta = \bz} \right] \bu / p_1(\bar{\bg}_t|\zeta) \; d\bar{\bg}_t \nonumber \\
&= \int_{\bar{\bg}_t} [p_{1,\bu}'(\bar{\bg}_t| \zeta) \vert_{\zeta = \bz}]^2 / p_1(\bar{\bg}_t | \bz) d\bar{\bg}_t, \label{eq:mixture_identity}
\end{align}
where $p_{1,\bu}'(\bar{\bg}_t| \zeta)$ denotes the directional derivative of $p_1(\bar{\bg}_t| \zeta)$ in the direction $\bu$. A similar identity holds for $\calI_{\bar{\bg}_t}(\bz | \calB_t^2)$ and $\calI_{\bar{\bg}_t}(\bz)$. For any $\bar{\bg}_t \in \mathbb{R}^p$, Equation 7 in \cite{cohen1968fisher} shows that
\begin{equation*}
\frac{[\lambda p_{1,\bu}'(\bar{\bg}_t| \zeta) + (1 - \lambda) p_{2,\bu}'(\bar{\bg}_t| \zeta)]^2}{\lambda p_1(\bar{\bg}_t| \zeta) + (1 - \lambda) p_2(\bar{\bg}_t| \zeta)} \leq \lambda \frac{[p_{1,\bu}'(\bar{\bg}_t| \zeta)]^2}{p_1(\bar{\bg}_t| \zeta)} + (1 - \lambda) \frac{[p_{2,\bu}'(\bar{\bg}_t| \zeta)]^2}{p_2(\bar{\bg}_t| \zeta)},
\end{equation*}
which follows by expanding the square and simple algebraic manipulations.
Integrating over $\bar{\bg}_t$ gives that $\bu^\top \calI_{\bar{\bg}_t}(\bz) \bu \leq \lambda \bu^\top \calI_{\bar{\bg}_t}(\bz | \calB_t^1) \bu + (1-\lambda) \bu^\top \calI_{\bar{\bg}_t}(\bz | \calB_t^2) \bu$, from which we obtain the desired result since $\bu$ was arbitrary. Now consider the uniform distribution over $B$-subsets of $\calD_\text{train}$, \emph{i.e.}, $\mathbb{P}(\calB_t) = 1/{n \choose B}$ for all $\calB_t \subseteq \calD_\text{train}, |\calB_t| = B$. The distribution of $\bar{\bg}_t$ is a mixture of ${n \choose B}$ distributions corresponding to each possible $\calB_t$. Applying \autoref{eq:mixture_fi} recursively gives the first bound.

For the second bound, denote by $p_{\calB_t}$ the density function of the noisy gradient when the batch is $\calB_t$, and by $p_{\calB_t, \bu}'$ its directional derivative in the direction $\bu$. Then by \autoref{eq:mixture_identity}:
\begin{align}
\bu^\top \calI_{\bar{\bg}_t}(\bz) \bu
&= \frac{\left[ \sum_{\calB_t \subseteq \calD_\text{train}:|\calB_t| = B} p_{\calB_t, \bu}'(\bar{\bg}_t | \zeta) / {n \choose B} \right]^2}{\sum_{\calB_t \subseteq \calD_\text{train}:|\calB_t| = B} p_{\calB_t}(\bar{\bg}_t | \zeta) / {n \choose B}} \nonumber \\
&= \frac{\left[ q \sum_{\calB_t \subseteq \calD_\text{train}:|\calB_t| = B,\bz \in \calB_t} p_{\calB_t, \bu}'(\bar{\bg}_t | \zeta) / {n-1 \choose B-1} \right]^2}{\sum_{\calB_t \subseteq \calD_\text{train}:|\calB_t| = B} p_{\calB_t}(\bar{\bg}_t | \zeta) / {n \choose B}} \nonumber \\
&= \frac{q^2 \left[ \sum_{\calB_t \subseteq \calD_\text{train}:|\calB_t| = B,\bz \in \calB_t} p_{\calB_t, \bu}'(\bar{\bg}_t | \zeta) / {n-1 \choose B-1} \right]^2}{\sum_{\calB_t \subseteq \calD_\text{train}:|\calB_t| = B,\bz \in \calB_t} p_{\calB_t}(\bar{\bg}_t | \zeta) / {n-1 \choose B-1}} \frac{\sum_{\calB_t \subseteq \calD_\text{train}:|\calB_t| = B,\bz \in \calB_t} p_{\calB_t}(\bar{\bg}_t | \zeta) / {n-1 \choose B-1}}{\sum_{\calB_t \subseteq \calD_\text{train}:|\calB_t| = B} p_{\calB_t}(\bar{\bg}_t | \zeta) / {n \choose B}} \label{eq:q_sq_bound}.
\end{align}
In the second term, for any $\calB_t$ not containing $\bz$, let $\calB_t^{(j)}$ be $\calB_t$ with its $j$-th element replaced by $\bz$ for $j=1,\ldots,B$. Since $\calB_t$ and $\calB_t^{(j)}$ differ in a single element, by the DP assumption we have that $e^{-\epsilon} p_{\calB_t^{(j)}}(\bar{\bg}_t | \zeta) \leq p_{\calB_t}(\bar{\bg}_t | \zeta)$ for all $j$, hence $e^{-\epsilon} \sum_{j=1}^B p_{\calB_t^{(j)}}(\bar{\bg}_t | \zeta) / B \leq p_{\calB_t}(\bar{\bg}_t | \zeta)$, giving:
\begin{align*}
\sum_{\calB_t \subseteq \calD_\text{train}:|\calB_t| = B} p_{\calB_t}(\bar{\bg}_t | \zeta) / {n \choose B} &= \left[ \sum_{\calB_t \subseteq \calD_\text{train}:|\calB_t| = B, \bz \in \calB_t} p_{\calB_t}(\bar{\bg}_t | \zeta) + \sum_{\calB_t \subseteq \calD_\text{train}:|\calB_t| = B, \bz \notin \calB_t} p_{\calB_t}(\bar{\bg}_t | \zeta) \right] / {n \choose B} \\
&\geq \left[ \sum_{\calB_t \subseteq \calD_\text{train}:|\calB_t| = B, \bz \in \calB_t} p_{\calB_t}(\bar{\bg}_t | \zeta) + \sum_{\calB_t \subseteq \calD_\text{train}:|\calB_t| = B, \bz \notin \calB_t} \sum_{j=1}^B e^{-\epsilon} p_{\calB_t^{(j)}}(\bar{\bg}_t | \zeta) / B \right] / {n \choose B} \\
&\stackrel{(*)}{=} \left[ \sum_{\calB_t \subseteq \calD_\text{train}:|\calB_t| = B, \bz \in \calB_t} p_{\calB_t}(\bar{\bg}_t | \zeta) + \frac{n-B}{B} e^{-\epsilon} \sum_{\calB_t \subseteq \calD_\text{train}:|\calB_t| = B, \bz \in \calB_t} \sum_{j=1}^B p_{\calB_t}(\bar{\bg}_t | \zeta) \right] / {n \choose B} \\
&= \left( \frac{B}{n} + \frac{n-B}{n} e^{-\epsilon} \right) \left( \sum_{\calB_t \subseteq \calD_\text{train}:|\calB_t| = B,\bz \in \calB_t} p_{\calB_t}(\bar{\bg}_t | \zeta) / {n-1 \choose B-1} \right) \\
&= (q + (1-q) e^{-\epsilon}) \left( \sum_{\calB_t \subseteq \calD_\text{train}:|\calB_t| = B,\bz \in \calB_t} p_{\calB_t}(\bar{\bg}_t | \zeta) / {n-1 \choose B-1} \right),
\end{align*}
where $(*)$ uses the fact that each $\calB_t$ containing $\bz$ appears in exactly $n-B$ of the $\calB_t^{(j)}$'s. Substituting this bound into the second term in \autoref{eq:q_sq_bound} gives an upper bound of $1/(q + (1-q) e^{-\epsilon})$, hence:
\begin{align*}
\bu^\top \calI_{\bar{\bg}_t}(\bz) \bu &\leq \int_{\bar{\bg}_t} \frac{q^2}{q + (1-q)e^{-\epsilon}} \frac{\left[ \sum_{\calB_t \subseteq \calD_\text{train}:|\calB_t| = B,\bz \in \calB_t} p_{\calB_t, \bu}'(\bar{\bg}_t | \zeta) / {n-1 \choose B-1} \right]^2}{\sum_{\calB_t \subseteq \calD_\text{train}:|\calB_t| = B,\bz \in \calB_t} p_{\calB_t}(\bar{\bg}_t | \zeta) / {n-1 \choose B-1}} d\bar{\bg}_t \\
&= \frac{q^2}{q + (1-q)e^{-\epsilon}} \bu^\top \calI_{\bar{\bg}_t}(\bz | \bz \in \calB_t) \bu.
\end{align*}
Since this holds for any $\bu \in \mathbb{R}^p$, we get that $\calI_{\bar{\bg}_t}(\bz) \preceq \frac{q^2}{q + (1-q) e^{-\epsilon}} \calI_{\bar{\bg}_t}(\bz | \bz \in \calB_t)$. Finally, assuming that the gradient of a sample is independent of other elements in the batch, we have that by the convexity of Fisher information (\autoref{eq:mixture_fi}):
$$q \calI_{\bar{\bg}_t}(\bz | \bz \in \calB_t) \preceq q \mathbb{E}_{\calB_t}[\calI_{\bar{\bg}_t}(\bz | \calB_t) | \bz \in \calB_t] = \mathbb{E}_{\calB_t}[\calI_{\bar{\bg}_t}(\bz | \calB_t) | \bz \in \calB_t] \mathbb{P}(\bz \in \calB_t) = \mathbb{E}_{\calB_t}[\calI_{\bar{\bg}_t}(\bz | \calB_t)],$$ so $\calI_{\bar{\bg}_t}(\bz) \preceq \frac{q}{q + (1-q) e^{-\epsilon}} \mathbb{E}_{\calB_t}[\calI_{\bar{\bg}_t}(\bz | \calB_t)]$.
\end{proof}

\section{Additional Details}
\label{sec:additional_details}

\paragraph{Model architectures.} In \autoref{sec:neural_network}, we trained two small ConvNets on the MNIST and CIFAR-10 datasets. We adapted the model architectures from \cite{papernot2020tempered}, using $\tanh$ activation functions and changing all max pooling to average pooling so that the loss is a smooth function of the input. For completeness, we give the exact architecture details in \autoref{tab:small_convnet} and \autoref{tab:large_convnet}.

\begin{table}[h!]
    \begin{minipage}{.5\linewidth}
      \centering
        \resizebox{\textwidth}{!}{
        \begin{tabular}{ll}
            \toprule
            Layer & Parameters \\
            \midrule
            Convolution $+\tanh$ & 16 filters of $8 \times 8$, stride 2, padding 2 \\
            Average pooling & $2 \times 2$, stride 1 \\
            Convolution $+\tanh$ & 32 filters of $4 \times 4$, stride 2, padding 0 \\
            Average pooling & $2 \times 2$, stride 1 \\
            Fully connected $+\tanh$ & 32 units \\
            Fully connected $+\tanh$ & 10 units \\
            \bottomrule
        \end{tabular}
        }
        \caption{Architecture for MNIST model.}
        \label{tab:small_convnet}
    \end{minipage}%
    \begin{minipage}{.5\linewidth}
      \centering
        \resizebox{\textwidth}{!}{
        \begin{tabular}{ll}
            \toprule
            Layer & Parameters \\
            \midrule
            (Convolution $+\tanh$)$\times 2$ & 32 filters of $3 \times 3$, stride 1, padding 1 \\
            Average pooling & $2 \times 2$, stride 2 \\
            (Convolution $+\tanh$)$\times 2$ & 64 filters of $3 \times 3$, stride 1, padding 1 \\
            Average pooling & $2 \times 2$, stride 2 \\
            (Convolution $+\tanh$)$\times 2$ & 128 filters of $3 \times 3$, stride 1, padding 1 \\
            Average pooling & $2 \times 2$, stride 2 \\
            Fully connected $+\tanh$ & 128 units \\
            Fully connected $+\tanh$ & 10 units \\
            \bottomrule
        \end{tabular}
        }
        \caption{Architecture for CIFAR-10 model.}
        \label{tab:large_convnet}
    \end{minipage} 
\end{table}

\paragraph{Hyperparameters.} Private SGD has several hyperparameters, and we exhaustively test all setting combinations to produce the scatter plots in \autoref{fig:mse_comparison}. \autoref{tab:hyp_mnist} and \autoref{tab:hyp_cifar10} give the choice of values that we considered for each hyperparameter.

\begin{table}[h!]
    \begin{minipage}{.45\linewidth}
      \centering
        \begin{tabular}{ll}
            \toprule
            Hyperparameter & Values \\
            \midrule
            Batch size & $600$ \\
            Momentum & $0.5$ \\
            \# Iterations $T$ & $1000,2000,3000,5000$ \\
            Noise multiplier $\sigma$ & $0.2, 0.5, 1, 2, 5, 10$ \\
            Step size $\rho$ & $0.01, 0.03, 0.1$ \\
            Gradient norm clip $C$ & $1, 2, 4, 8, 16, 32$ \\
            \bottomrule
        \end{tabular}
        \caption{Hyperparameters for MNIST model.}
        \label{tab:hyp_mnist}
    \end{minipage}%
    \begin{minipage}{.55\linewidth}
      \centering
        \begin{tabular}{ll}
            \toprule
            Hyperparameter & Values \\
            \midrule
            Batch size & $200$ \\
            Momentum & $0.5$ \\
            \# Iterations $T$ & $12500,18750,25000,31250,37500$ \\
            Noise multiplier $\sigma$ & $0.1, 0.2, 0.5, 1, 2$ \\
            Step size $\rho$ & $0.01, 0.03, 0.1$ \\
            Gradient norm clip $C$ & $0.1, 0.25, 0.5, 1, 2, 4, 8, 16$ \\
            \bottomrule
        \end{tabular}
        \caption{Hyperparameters for CIFAR-10 model.}
        \label{tab:hyp_cifar10}
    \end{minipage} 
\end{table}